\documentclass[preprint,12pt]{elsarticle}
\usepackage{xcolor}
\usepackage{amsmath}
\usepackage{amssymb}
\usepackage{algorithm}
\usepackage{algorithmic}
\usepackage{textcomp}
\usepackage{tikz}
\usepackage{subcaption}
\usepackage{lscape}
\usetikzlibrary{arrows,shapes,shadows,positioning,fadings,decorations,decorations.pathmorphing}
\tikzstyle{block}=[draw opacity=0.7,line width=1.4cm]
\newcommand{\node}{v}

\newcommand{\labelfun}{L}

\newcommand{\hash}{\pi}
\newcommand{\subtreerootedat}[1]{\raisebox{-0.1cm}{$\stackrel{#1}{\triangle}$}}
\newcommand{\subtreelimitedrootedat}[2]{\raisebox{-0.1cm}{$\stackrel{#1}{\triangle}\!\!{\text{\textbar}_{#2}}$}}
\newcommand{\nuovotk}{{ST+}}

\newcommand{\depth}{depth}
\newcommand{\visit}[1]{\subtreerootedat{#1}}
\newcommand{\limitedvisit}[2]{\subtreelimitedrootedat{#1}{#2}}

\newtheorem{theorem}{Theorem}[section]
\newenvironment{proof}[1][Proof]{\begin{trivlist}\item[\hskip \labelsep {\bfseries #1}]}{\end{trivlist}}

\journal{Neurocomputing}
\begin{document}
\begin{frontmatter}

\title{Ordered Decompositional DAG Kernels Enhancements}

\author[qcri]{Giovanni Da San Martino}
  \ead{gmartino@qf.org.qa}
\author[unipd]{Nicol\`o Navarin\corref{cor1}}
  \ead{nnavarin@math.unipd.it}
\author[unipd]{Alessandro Sperduti} 
  \ead{sperduti@math.unipd.it}

\cortext[cor1]{Corresponding author}
\address[qcri]{Qatar Computing Research Institute, HBKU, P.O. Box 5825 Doha, Qatar}
\address[unipd]{Department of Mathematics, University of Padova, via Trieste 63, Padova, Italy}

\begin{abstract}
In this paper, we show how the Ordered Decomposition DAGs (ODD) kernel framework, 
a framework that allows the definition of graph kernels from tree kernels, allows to easily define new state-of-the-art graph kernels. 
Here we consider a fast graph kernel based on the Subtree kernel (ST), and we propose various enhancements to increase its expressiveness. 
The proposed DAG kernel has the same worst-case complexity as the one based on ST, but an improved expressivity due to an augmented set of features.
Moreover, we propose a novel weighting scheme for the features, which can be applied to other kernels of the ODD framework. 
These improvements allow the proposed kernels to improve on the classification performances of the ST-based kernel for several real-world datasets, reaching state-of-the-art performances. 

\end{abstract}


\end{frontmatter}

\section{Introduction} \label{sec:introduction}

The increasing availability of data in structured form, such as trees~\cite{Denoyer2007} or graphs~\cite{dobson2003, springerlink:10.1007/s10115-007-0103-5, Weislow1989}, has led to the development of machine learning techniques able to deal directly with such types of data. 
Among these, kernel methods, such as Support Vector Machines (SVM) \cite{Taylor-Cristianini:Book2004}, have become very popular due to their generalization ability and state of the art performances in many tasks, such as relationship extraction~\cite{Simoes2013}, analysis of RDF data~\cite{Vries2013a}, action recognition~\cite{Wang2013a}, text categorization of biomedical data~\cite{Bleik2013} and bioinformatics~\cite{Kundu2013}. 

The class of kernel methods comprises all those learning algorithms which do not require an explicit representation of the input, but only information about the similarities among them. 
A simple way of assessing the similarity between two objects described by a set of features is to compute the dot product of their representation in feature space. 
If a ``similarity'' function $k(\cdot,\cdot)$, corresponding to a dot product $\langle \cdot ,\cdot\rangle$ in feature space, is available,
the intermediate step of explicitly representing the data can be avoided. In fact, computing $k(x_1,x_2)$  implicitly corresponds to perform a 
nonlinear transformation of the input vectors $x_1$ and $x_2$ via a function $\phi(\cdot)$ and then to compute the dot product of the resulting vectors $\phi(x_1)$ and $\phi(x_2)$. The function $\phi(\cdot)$ projects the input vectors into a feature space of much higher (possibly infinite) dimension where it is more likely to accomplish the learning task. 
Kernel methods generally formulate a learning problem as a constrained optimization one, where an objective function combining an empirical risk term with a (quadratic) regularizer must be minimized. If the employed kernel function is symmetric positive semidefinite, the problem is convex and thus has a global minimum \citep{Taylor-Cristianini:Book2004}. 

Any kernel method can be decomposed into two modules: {\it i)} a problem specific kernel function; {\it ii)} a general purpose learning algorithm (the solver). 
Since the solver interfaces with the problem only by means of the kernel function, it can be used with any kernel function, and vice-versa. Examples of popular kernel methods are the perceptron \citep{doi:10.1137/S0097539703432542} for the on-line setting, and the Support Vector Machines \citep{Taylor-Cristianini:Book2004} for the batch setting. Note that, provided an appropriate kernel function is given, any kernel method can be applied to any type of data. More importantly, the kernel function encodes all the information about the input data, thus the definition of appropriate kernel functions is crucial for the outcome of the learning algorithm.

A popular strategy for defining kernel functions for structured data is to decompose the structures into their constituent parts, and then, for each pair of parts, apply a local kernel \cite{Haussler1999}. 
While this strategy has been proved successful for strings and trees \cite{Collins2002, Vishwanathan2003, Moschitti2006a, Aiolli2009, Aiolli2011, DBLP:conf/icann/BacciuMS12}, it is not directly applicable to graphs because of the computational complexity issues which arise: representing a graph in terms of its subgraphs is not feasible since subgraph isomorphism, an {\it NP-complete} problem, should be solved for each pair of subgraphs. In \cite{Gartner2003a} it has been demonstrated that, any kernel whose feature space mapping is injective, is as hard to compute as graph isomorphism, an {\it NP} problem that still is not known whether it is in {\it P} or if it is {\it NP-complete}. 
Due to this limitation, the available strategies for building kernels are: \textit{i}) restricting the input domain to a class of graphs for which isomorphism can be checked quickly \cite{Schietgat2009}; \textit{ii}) select a priori a set of features, usually corresponding to a specific type of substructure, such as walks \cite{Gartner2003a}, paths \cite{Costa2010, Suard2007}, subtree patterns \cite{Mah'e2009, Shervashidze2009a}. 
 The former approach can be applied to a limited type of graphs, the latter tends to have a limited flexibility: when the available kernels are not relevant to the task, a new one has to be designed. However, defining an efficient symmetric positive semidefinite kernel, corresponding to the desired feature space, can be an extremely difficult task. 
 All the above approaches discard information about the original graph and are effective only when the selected features are relevant for the current problem. 
We propose to design graph kernels as follows: first transform the graphs into simpler structures, i.e. multisets of directed acyclic graphs (DAGs), and then extend the definition of a large class of already available kernels for trees to DAGs. 
Our approach allows the application of the vast literature on kernels for trees, which consists of fast and/or very expressive kernels, to the graph domain. 

{Generally speaking, a serious drawback which prevents many of the kernels listed above to be applied to large datasets is their computational time complexity. 
Those kernels which can be applied to large datasets exploit a ``limited'' number of features to represent a graph. For example, the kernel proposed in  \cite{Shervashidze2009a} has a linear complexity in the number of edges of the graphs because any graph is represented in the feature space by a number of non-zero features which is proportional to the number of nodes of the graph. On the other hand, a too compact representation of a graph in feature space may have a negative impact on the effectiveness of the kernel because of a reduced discrimination ability.}  

{In this paper, we tackle this problem by proposing various enhancements to a fast graph kernel based on the Subtree kernel for trees (ST)~\cite{Collins2001}.} 
Among these, the main contribution is a novel tree kernel, which has the same worst-case complexity of the ST kernel, while the size of its feature space is much larger.  

The paper is structured as follows. 
Section~\ref{sec:definitions} introduces some basic notation and definitions. 
Section~\ref{sec:background} recalls the ODD framework, of which the proposed kernels are instances. 
Section~\ref{sec:dagkernels} describes the main contributions of the paper: the $\nuovotk$ kernel for DAGs and a novel weighting scheme for the features, which can be applied to other kernels of the ODD framework. 
Section~\ref{sec:relatedwork} discusses some related kernels for graphs, and Section~\ref{sec:exps} provides experimental evidence of the effectiveness of the proposed approaches. 
Finally, Section~\ref{sec:conclusions} draws conclusions. 

{The paper extends the work in \cite{DaSanMartino2015} by adding: {\it i)} a self-contained and simplified description of the $\nuovotk$ kernel; {\it ii)} a novel, more effective, feature weighting scheme; {\it iii)} an extended and revised ``Related Work'' section; {\it iv)}  a novel set of experiments which are now performed on much larger benchmark datasets and for a larger number of competing graph kernels; {\it v)} a comparison among  empirical execution times of the various experimented kernels.} 

\section{Notation} \label{sec:definitions}
 
A graph is a triplet $G=(V,E,L)$, where $V$ (alternatively $V_G$) is the set of nodes ($|V|$ is the number of nodes), $E$ the set of edges and $\labelfun()$ a function returning the label of a node. All labels are obtained from a fixed alphabet $\cal{A}$. 
A graph is undirected if $(\node_i,\node_j)\in E \Leftrightarrow (\node_j,\node_i)\in E$, otherwise it is directed. 
{A path in a graph is a sequence of nodes $\node_1,\ldots,\node_n$ such that $\node_i \in V, 1\leq i\leq n$, \mbox{$(\node_i,\node_{i+1})\in E$} and 
$\forall 1\leq i\leq n, 1\leq j<n, j\neq i. \node_i\neq\node_j$ (no node, except the first one, can appear twice in the same path).} 
A cycle is a path for which $\node_1=\node_n$; 
a cycle is even/odd if its number of nodes is even/odd, respectively. 
A graph is connected if there exists a path connecting each pair of nodes. 
A connected graph is rooted if exactly one node has no incoming edges. 
A graph is ordered if the set of neighbours of each node is ordered.  
A tree is a rooted connected directed acyclic graph where each node has at most one incoming edge. 
A subtree of a tree $T$ is a connected subset of nodes of $T$. 
A proper subtree is a subtree composed by a node and all of its descendants. 
Given a node $\node$ of a tree, $\rho(\node)$ represents the outdegree of $\node$, i.e. the number of nodes connected to $\node$. We will use $\rho$ as the maximum outdegree of a node in either a tree or a graph. 
The depth $\depth(\node)$ of a node $\node$ is the number of edges in the shortest path between the root of the tree and $\node$. 
If the tree is ordered, $ch_{\node}[j]$ represents the $j$-th child of $\node$ and $chs_{\node}[j_1,j_2,\ldots,j_n]$ indicates the set of children of $\node$ with indices $j_1, j_2,\ldots,j_n$. 
{Given a graph $G$ and a node $\node\in V(G)$, we define a subtree-walk of size $h$ as the tree obtained by the following procedure: the root of the tree is $\node$; at each step $i$, with $1\leq i\leq h$, and for each current leaf node $\node_j$ of the tree, any neighbouring node of $\node_j$ in $G$ is added to the tree as a child of $\node_j$. Note that, when $h>1$, typically a node of $G$ can appear multiple times in the same subtree-walk. Given a DAG $D$ and a node $\node_i\in V(D)$, we define a tree-visit, denoted by $\visit{\node_i}$, as the tree resulting from the visit of $D$ starting from the node $\node_i$. Such visit returns all the nodes of $D$ reachable from $\node_i$.  
 If a node $\node_{j}$ can be reached more than once, more occurrences of $\node_{j}$ will appear in $\visit{\node_i}$ (see Figure~\ref{fig:nuovokernelfeaturespace}-b for an example).}  

\section{Preprocessing: from Graphs to Multisets of DAGs} \label{sec:background}

This section recalls the ODD-Kernels framework for graphs~\cite{Dasan2012}.  
The idea of our approach is to transform the graphs into simpler structures, i.e. DAGs, and then apply a kernel for such structures. 
The following subsections explain each step of the transformation. 
\subsection{From Graph to DAGs}
  The graph $G$ is mapped into a multiset of DAGs $DD_{G}=\{DD_{G}^{\node_{i}} | \node_{i} \in V_{G} \}$, where $DD_{G}^{\node_{i}}=(V_G^{\node_i}, E_G^{\node_i}, L)$ is obtained by keeping each edge in the shortest path(s) connecting $\node_i$ with any $\node_j\in V_G$. 
  From a practical point of view, \mbox{$DD_G^{\node_i}$} can be built by performing a breadth-first visit on the graph $G$ starting from node $\node_i$ and 
  applying the following rules: 
\begin{enumerate}
	\item during the visit a direction is given to each edge; if $\node_j$ is reached from $\node_i$ in one step, then $(\node_i,\node_j)\in E_G^{v_i}$ (note that edge $(\node_j,\node_i)$ is not added to $E_G^{v_i}$);
  \item edges connecting nodes reached at level $l$ of the visit to nodes reached at level $g<l$ are not added to $E_G^{v_i}$ (such edges would induce a cycle in $DD_G^{\node_i}$.) 
\end{enumerate}
For every choice of $G$ and $\node_{i}$, a single \textit{Decompositional Dag} $DD^{\node_i}_G$ is generated. 
 By repeating the procedure for each node of the graph, $|V|$ DAGs are obtained. 
   Figure~\ref{fig:dagconstruction} shows the four $DD$s obtained from the undirected graph in Figure~\ref{fig:dagconstruction}-a.  
  Note that when the same node is reached simultaneously (at the same level of the visit) from different nodes, then all involved edges are preserved. For example,  when considering the visit at level $2$ starting from node \textbf{s}, the 
   node \textbf{d} is reached simultaneously by edges $(\textbf{b},\textbf{d})$ and $(\textbf{e},\textbf{d})$, and both of them are preserved in the corresponding Decompositional DAG (see Figure~\ref{fig:dagconstruction}-b).  
\begin{figure}[ht]
\centering
\begin{tikzpicture}[auto,thick,scale=.5]
   \tikzstyle{graph}=[scale=1.0,minimum size=12pt, inner sep=0pt, outer sep=0pt,ball color=blue,circle, text=white]
		\node (s) [xshift=0.5cm,graph] {d}   
					child {
            node (b) [graph] {b}
        	};
         \node (e) [graph] at (0:3){e}   
         		child { node (d) [graph] {s}};		
    ;    	
    \draw
    (s) to (e)
    (b) to (d)
    (b) to (e)
    ;
  \small
	 \draw [red,->,line width=5pt] (3.8,-0.8)--(5.6,-0.8);
	 \node [yshift=-0.40cm] {a)};
	 \node [yshift=0.5cm, xshift=3.25cm] {b)};
	 \node [yshift=0.5cm, xshift=5.6cm] {c)};
	 \node [yshift=-1.0cm, xshift=3.25cm] {d)};
	 \node [yshift=-1.0cm, xshift=5.6cm] {e)};
   \tikzstyle{node}=[minimum size=12pt, inner sep=0pt, outer sep=0pt,ball color=blue,circle, text=white]
        \node [node,xshift=4.2cm, yshift=1.3cm] {\textbf{s}} [->] 
                child {
                    node [node] {\textbf{e}}[] 
                    child {node (d2)[node, xshift=0.35cm]{\textbf{d}}[]}
                }
                child {
                    node (b2) [node]{\textbf{b}}[]
                }
				;
         \draw [->]
         (b2) to (d2)
        ;
        \node (e) [node,xshift=6.8cm,yshift=1.3cm] {e}
        		child[->] {
            			node (s) [node] {s}
            			}
						child[->]{
									node (b) [node] {b}
									}
	   				child[->] {
            			node (d) [node] {d}
            			}
        ;
        \node [node,xshift=4.2cm,yshift=-1.0cm] {b}
        		child[->] {
            			node [node] {s}
            			}
						child[->]{
									node [node] {e}
									}
	   				child[->] {
            			node [node] {d}
            			}
        ;    
        \node [node,xshift=6.8cm, yshift=-0.2cm] {\textbf{d}} [->] 
                child {
                    node [node] {\textbf{b}}[] 
                    child {node (d2)[node, xshift=0.35cm]{\textbf{s}}[]}
                }
                child {
                    node (e2) [node]{\textbf{e}}[]
                }
				;
         \draw [->]
         (e2) to (d2)
        ;    
\end{tikzpicture}
\caption{Example of decomposition of a graph a) into its $4$ DDs b-e).\label{fig:dagconstruction}}
\end{figure}
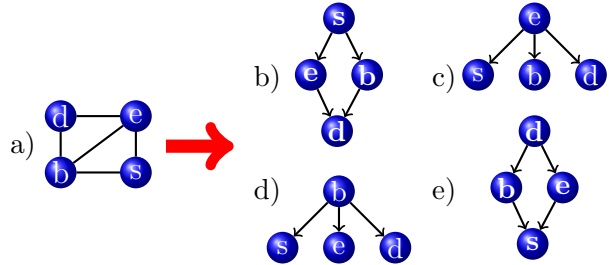
  In order to reduce the total number of nodes of $DD_G^v$, we propose to limit the depth of the visits during the generation of the multiset of DAGs ~\cite{Dasan2012} to a constant value $h$.
\textcolor{black}{The resulting DAG will be referred to as $DD_G^{v,h}$.}
  Given $\node\in V_G$, let $H$ be the number of nodes generated by the visits up to depth $h$. 
  An upper bound for $H$ is $\rho^h$. Notice, this is a loose bound, in many practical cases. 
  The total number of nodes  of $DD_G$ is $|V_G|H$. Note that, if $\rho$ is constant, then also $H$ is constant. 

\subsection{Ordering DAG nodes}
  The kernels we define in the following, which are all straightforwardly derived from tree kernels, require DAG nodes to be ordered. 
  Therefore, we define a strict partial order $\dot{<}$ between DAG nodes in $DD_{G}^{\node_{i}}$ obtaining Ordered DAGs $ODD_{G}^{\node_{i}}$. 
  The ordering makes use of a unique representation of subtrees as strings inspired by \cite{Vishwanathan2003}. 
  Here we modify such mapping by employing perfect hash functions{, i.e. hash functions which guarantee to have no collisions,} to encode subtrees~\cite{Shervashidze2009a, DaSanMartino2012}. 
  Let $\kappa()$ be a perfect hash function, $\#,\lceil,\rfloor$ be symbols never appearing in any node label and $ch_\node[j]$ the $j$-th node in the ordered sequence of outcoming edges of $\node$, then 
\begin{equation}
\hash(\visit{\node})=\begin{cases}
  \kappa(L(\node)) \text{\hspace{3.2cm}if $\node$ is a leaf node}\\
  \kappa\big(L(\node)\big\lceil \hash(\visit{ch_\node[1]})\#\ldots\#\hash(\visit{ch_\node[\rho(\node)]})\big\rfloor\big) \text{\hspace{0.1cm}otherwise}\\ 
  \end{cases}
  \label{eq:orderingfunction}
 \end{equation}
 where the children of $\node$ are recursively ordered according to their $\hash()$ values. To simplify notation, in the following, when it is
 clear from the context, we will use the notation $\hash(\node)$ instead of $\hash(\visit{\node})$.
Then $\node_{i} \dot{<} \node_{j}$ if $\hash(\node_i)<\hash(\node_j)$, where $<$ is the relation of order between alphanumeric strings. 
Notice that $\hash(\node_i)=\hash(\node_j)\Leftrightarrow \neg(\node_i \dot{<} \node_j) \wedge \neg(\node_j \dot{<} \node_i)$, i.e.  $\hash(\node_i)=\hash(\node_j)$ if and only if the nodes $\node_i$ and $\node_j$ are not comparable. 
In such case, many orderings for non comparable children nodes according to $\dot{<}$
 are possible. 
 We are now going to prove some results that will make it easier to show, in Section~\ref{sec:dagkernels}, that each kernel described in this paper (as well as for a large class of kernels for trees) yield the same features, independently of the ordering of non comparable nodes. 
 Since all the features of the kernels in Section~\ref{sec:dagkernels} are extracted from tree visits of DAG nodes, 
 our goal here is to show that isomorphic DAGs yield the same tree visits. 
 We first show that if two DAGs $DD_{G_1}^{\node_i}$ and $DD_{G_2}^{\node_j}$ are isomorphic, then the root nodes of the DAGs are not comparable with respect to the ordering $\dot{<}$, in fact:
\begin{theorem} \label{teo:dagisomorfi}
	if two DAGs $DD_{G_1}^{\node_i}$ and $DD_{G_2}^{\node_j}$ are isomorphic, then\\ \mbox{$\neg(\node_i \dot{<} \node_j) \wedge \neg(\node_j \dot{<} \node_i)$}. 
\end{theorem}
\begin{proof}
  Let \mbox{$f: V_{G_1} \rightarrow V_{G_2}$} be an isomorphism between $DD_1^{\node_i}$ and $DD_2^{\node_j}$. 
  We prove the thesis by induction. Let $f(\node_i)=\node_j$, since the nodes are isomorphic $L(\node_i)=L(\node_j)$. If $\node_i$ and $\node_j$ are leaf nodes, then $\hash(\node_i)=\hash(\node_j)$ and consequently $\neg(\node_i \dot{<} \node_j) \wedge \neg(\node_j \dot{<} \node_i)$. Otherwise, by inductive hypothesis $\forall l. 1\leq l\leq \rho(\node_i).$  $\hash(ch_{\node_i}[l])=\hash(ch_{f(\node_i)}[l])$ and $L(\node_i)=L(f(\node_i))$, thus $\hash(\node_i)=\hash(f(\node_i))=\hash(\node_j)$. 
\end{proof}
%
%
%
The following theorem shows that two non comparable nodes $\node_i, \node_j$, yield identical tree visits $\visit{\node_i}$, $\visit{\node_j}$: 
\textcolor{black}{
\begin{theorem} \label{teo:visits}
	Given the ordering $\dot{<}$, $\neg(\node_i \dot{<} \node_j) \wedge \neg(\node_j \dot{<} \node_i)$ if and only if 
	 $\subtreerootedat{\node_i}$ and $\subtreerootedat{\node_j}$ 
	 are identical.
\end{theorem}
\begin{proof}
 If $\neg(\node_i \dot{<} \node_j) \wedge \neg(\node_j \dot{<} \node_i)$ then $\hash(\node_i)=\hash(\node_j)$. Recalling that $\kappa()$, the function on which $\hash()$ is based on, is a perfect hash function, we prove the thesis by induction. If $\node_i, \node_j$ are leaf nodes, then $\hash(\node_i)=\hash(\node_j)\Leftrightarrow L(\node_i)=L(\node_j)$. If $\node_i, \node_j$ are not leaf nodes, then $\forall l. 1\leq l\leq \rho(\node_i)$~ $\visit{ch_{\node_i}[l]}$ is identical to $\visit{ch_{\node_j}[l]}$ for inductive hypothesis, and then it must be $L(\node_i)=L(\node_j)$ since $\hash(\node_i)=\hash(\node_j)$; therefore $\visit{\node_i}$ is identical to $\visit{\node_j}$. 
  Now we show that if $\visit{\node_i}$ is identical to $\visit{\node_j}$, then $\hash(\node_i)=\hash(\node_j)$ by induction. The base case has already been proved by the equality $\hash(\node_i)=\hash(\node_j)\Leftrightarrow L(\node_i)=L(\node_j)$. By inductive hypothesis $\hash(ch_{\node_i}[m])=\hash(ch_{\node_j}[m])$ for each child $m$ of  $\node_i$ and $\node_j$. Then $\hash(\node_i)=\hash(\node_j)$ and $\neg(\node_i \dot{<} \node_j) \wedge \neg(\node_j \dot{<} \node_i)$. 
\end{proof}
}
Note that, since any ordering between non comparable vertices is equivalent for our goals, we avoid to give a specific ordering. 
If the $\hash()$ values are computed according to a post order visit of the DAG, then the values $\hash(ch_{\node}[l])$ for $1\leq l\leq \rho(\node)$ are already available when computing $\hash(\node)$. Thus the time complexity of the ordering phase of the DAG is $O(|V_G| \rho\log \rho)$ where the term $\rho\log \rho$ accounts for the ordering of the children of each node. 
\subsection{Applying Tree Kernels to DAGs}
If we restrict to the kernels which are going to be presented in this paper, the general formula for graph kernels derived from the ODD framework~\cite{Dasan2012} can be simplified as follows 
    \begin{equation}
        ODD_{K}(G_1,G_2) = \!\!\!\!\!\!\!\!\!\!\displaystyle \sum_{\substack{D_1 \in ODD_{G_1} \\D_2 \in ODD_{G_2}}} \langle \phi^{K}(D_1),\phi^{K}(D_2) \rangle, 
   \label{eq:k}
  \end{equation}
where $\langle \cdot,\cdot\rangle$ is the dot product operator, and $\phi^{K}(D)$ is the explicit feature space projection of the DAG $D$ with respect to the kernel $K$ and \mbox{$ODD_G=\{ODD_G^{v,h} | v \in V_G\}$}. 
{Section~\ref{sec:st} gives an example of an instance of the kernel defined in \eqref{eq:k}.}

\section{Kernels for DAGs} \label{sec:dagkernels}

In Section~\ref{sec:background}, we showed a preprocessing procedure for transforming a graph into a multiset of ordered DAGs. 
In this section, we first recall the ODD$_{ST_h}$ kernel, presenting it in a slightly different way than as it was originally introduced in the paper~\cite{Dasan2012}. 
Then, we describe the original contributions of the paper, i.e. a novel kernel for DAGs, named $\nuovotk$, and a novel weighting scheme for the features which is specifically designed for our setting. 

\subsection{ST kernel for DAGs} \label{sec:st}
 Let us consider $\visit{\node}$, the tree resulting from the visit of $ODD_{G}^{\node}$ starting from the root node $\node$. 
 The visit can be stopped when the tree $\visit{\node}$ reaches a maximum depth $h$. 
 Such tree is referred to as $\limitedvisit{\node}{h}$. 

 As an example of kernel in \eqref{eq:k}, we recall the ODD$_{ST_h}$ kernel~\cite{Dasan2012}.  
 The features of the kernel are $\limitedvisit{\node}{l}$, for each $\node\in V_D$, where $D\in ODD_G$ as defined in the previous section and for each $0 \leq l\leq h$. 
Specifically, any node $v$ of the DAG contributes to the feature vector $\phi(\cdot)$ as $\phi_{\hash(\node)} = \lambda^{\frac{size}{2}}$, where $size=|{\limitedvisit{\node}{l}}|$ for some $l$, and $\hash(\node)$ (we recall that this notation stays for $\hash(\limitedvisit{\node}{l})$) is the function defined by \eqref{eq:orderingfunction}. This weighting scheme for the features is inherited by the $ST$~\cite{Collins2001} kernel and it is motivated by the fact that 
when computing a kernel involving two matching large trees, the value returned by the kernel is very large because not only
the whole trees will match, but all their subtrees will match as well. To correct that, the contribution to the
kernel of a matching tree is down-weighted by $\lambda^{\frac{size}{2}}$, where $0< \lambda \leq 1$.

\textcolor{black}{
In order to demonstrate that the resulting graph kernel is positive semidefinite, we need to prove that our $\phi(\cdot)$ function is well-defined, i.e. it gives the same result when the representation of the input is changed without changing the value of the input.
If two graphs are isomorphic, they generate the same multiset of DAGs (since they are defined over shortest paths).
We know from Theorem~\ref{teo:dagisomorfi} that isomorphic DAGs generate the same visits.
Since the features considered by the ST kernel are subtrees, it directly follows from Theorem~\ref{teo:visits} that the swapping of non comparable vertices in the ordering do not affect the feature space representation of a graph. Thus, we provided a well-defined feature space representation for $ODD_{ST_h}$, from which it follows that the kernel is positive semidefinite.}

\subsection{The $\nuovotk$ Kernel for DAGs} \label{sec:noveltk}

The kernel we introduce in this section enlarges the feature space of the ST kernel, with a modest increase in computational burden, and is referred to as $\nuovotk$. 
In Algorithm~\ref{alg:nuovokernelfeatures} we define a procedure to compute the explicit feature space representation $\phi(\cdot)$ of $\nuovotk$.
\textcolor{black}{Note that this procedure accesses the graph only by means of  $\visit{\node}$ and $\limitedvisit{\node}{l}$, moreover
if two trees $\visit{\node_i}$ and $\visit{\node_j}$ are identical, than also all their subtrees are.
Thus, if two nodes generates the same $\hash(\node_i)=\hash(\node_j)$, then  $\visit{\node_i}=\visit{\node_j}$ and $\limitedvisit{ch_m[\node_i]}{l}=\limitedvisit{ch_m[\node_j]}{l}$ for each $m$ and $l$.
Thus, by Theorem~\ref{teo:visits} the procedure is well defined also in the presence of non-comparable nodes, since the resulting tree visits are the same.}
This proves that the kernel is positive semidefinite. 
\begin{algorithm*}[ht]
\centering
\scriptsize
\begin{algorithmic}[1]
    \STATE {\bfseries Input:} an ordered DAG $D$, the maximum depth of the visit $h$ 
    \FOR{{\bfseries each $\node\in V_D$}}
      \STATE $f = \visit{\node}$
      \STATE $\phi_{\hash(f)} = \phi_{\hash(f)} + \lambda^{\frac{|f|}{2}}$ // add the proper subtree rooted at $\node$ as a feature.
	\STATE \hspace{2.72cm}// if the feature is first encountered, it is assumed $\phi_{\hash(f)}=0$
\FOR{{\bfseries $0\leq l < \min(h,depth(f))$}}
	\FOR{{\bfseries $1\leq j\leq \rho(\node)$}}
 	  \STATE \begin{tikzpicture}[-,thick,scale=0.47, sibling distance=3.25cm]
 	          \node []{$\node$}
		   child {node (lft) []{$\limitedvisit{ch_1[\node]}{l}$}}
		   child {node[xshift=0.1cm]{$\ldots$}}
		   child {node[xshift=0.3cm]{$\limitedvisit{ch_{j-1}[\node]}{l}$}}
		   child {node[]{$\visit{ch_j[\node]}$}}
		   child {node[xshift=-0.3cm]{$\limitedvisit{ch_{j+1}[\node]}{l}$}}
		   child {node[]{$\ldots$}}
		   child {node (rgt)[]{$\limitedvisit{ch_{\rho(\node)}[\node]}{l}$}}	
		  ;
		  \node [left=0.0cm of lft, yshift=0.3cm] {$f^\prime = $};
 	         \end{tikzpicture} 
	  \STATE ~$\phi_{\hash(f^\prime)} = \phi_{\hash(f^\prime)} + \lambda^{\frac{|f^\prime|}{2}}$ // add the subtree $f^\prime$ as a feature.
 	\ENDFOR
     \ENDFOR
    \ENDFOR
    \STATE {\bfseries Output:} $\phi(\cdot)$, the set of features of $D$
\end{algorithmic}
\caption{A procedure for computing the features of the $\nuovotk$~kernel.\label{alg:nuovokernelfeatures}}
\end{algorithm*}
The set of features related to the $\nuovotk$~kernel is a superset of the features of ST and a subset of the features of PT~\cite{Moschitti2006a}. 
Line 8 of Algorithm~\ref{alg:nuovokernelfeatures} depicts a generic feature introduced by $\nuovotk$. 
Given a node $\node$ and an index $j$, the feature is defined as the subtree $\visit{\node}$ where all subtrees rooted at children of $\node$, except for the $j$-th child,
are replaced by a corresponding limited visit of $l$ levels. 
Notice that the feature actually depends on $\node\in V_D$, the index of a child $j$ and a limit $l$ on the depth of the visits.
The function $\pi(f)$ returns the index of the feature $f$ in $\phi(\cdot)$. 
Figure~\ref{fig:nuovokernelfeaturespace} depicts a partial feature space representation of a DAG according to $\nuovotk$. 
While for the ST kernel there is one feature for each $\node\in V_D$, $\nuovotk$~ associates at most 
$(\rho(\node) \cdot h)+1$ features for any $\node\in V_D$. 
\begin{figure}[t] 
\centering
\begin{tikzpicture}[auto,thick,scale=.5]
\tikzstyle{node}=[minimum size=12pt, inner sep=0pt, outer sep=0pt,ball color=blue,circle, text=white]
        \node (s) [node,xshift=2.0cm, yshift=1.3cm] {\textbf{s}} [->] 
                child {
                    node [node, xshift=-0.6cm] {\textbf{d}}[] 
		     child {
			node [node, xshift=-0.0cm]{\textbf{f}}[]
		    }
                }
                child {
		    node [node, ball color=yellow, text=black]{\textbf{v}}[]
  		    child {
			node [node, xshift=0.25cm]{\textbf{a}}[]
		    }
		    child {
		      node [node, ball color=yellow, text=black]{x}[]
		      child {
			node [node]{\textbf{d}}[]
			child {
			  node [node]{\textbf{e}}[]
			}
			child {
			  node [node]{\textbf{f}}[]
			}
		      }
		    }
                   child {
		      node (ei) [node, xshift=0.1cm]{\textbf{e}}[]
		      child {
			node (preterminale) [node, xshift=0.1cm]{\textbf{e}}[]
		      }
		      child {
			node [node, xshift=-0.0cm]{\textbf{f}}[]
			child {
			  node (fogliag) [node, xshift=-0.36cm]{\textbf{g}}[]
			}
		      }
		    }
		    child {
			node (eei) [node, xshift=0.2cm]{\textbf{e}}[]
			child {
			  node [node]{\textbf{g}}[]
			}
		    }
               }
		child {
		  node (b) [node, xshift=0.5cm]{\textbf{b}}[]
		};
	\draw[->] (preterminale) to (fogliag); 
	\node (laba) []{a)}[];
	\draw [red,->,line width=5pt, right=1cm of laba] (7.5,-0.4)--(9.0,-0.4);
		    \node (rf1) [node, right=4cm of s, yshift=-0.5cm]{\textbf{v}}[->]
  		    child {
			node [node, xshift=0.25cm]{\textbf{a}}[]
		    }
		    child {
		      node [node]{\textbf{x}}[]
		      child {
			node [node, xshift=-0.4cm]{\textbf{d}}[]
			child {
			  node [node]{\textbf{e}}[]
			}
			child {
			  node [node]{\textbf{f}}[]
			}
		      }
		    }
                   child {
		      node (ei) [node, xshift=0cm]{\textbf{e}}[]
		      child {
			node [node, xshift=0.1cm]{\textbf{e}}[]
			child {
			  node [node]{\textbf{g}}[]
			}
		      }
		      child {
			node [node, xshift=0.1cm]{\textbf{f}}[]
			child {
			  node [node]{\textbf{g}}[]
			}
		      }
		    }
		    child {
			node (eei) [node, xshift=0.2cm]{\textbf{e}}[]
			child {
			  node [node]{\textbf{g}}[]
			}
		    };
 	  \node (labb) [right=4.35cm of laba]{b)}[];
		    \node (rf2) [node, below=1.70cm of laba, xshift=0.7cm]{\textbf{v}}[->]
		    child {
		      node [node]{\textbf{x}}[]
		      child {
			node [node]{\textbf{d}}[]
			child {
			  node [node]{\textbf{e}}[]
			}
			child {
			  node [node]{\textbf{f}}[]
			}
		      }
		    }
		    ;
      \node (labc) [below=2.8cm of laba]{c)}[]; %
      \node[right=0.7cm of labc, yshift=-0.65cm]{$l=0$}[]; %
      \node [right=3.6cm of labc, yshift=-0.65cm]{$l=1$}[];
      \node [right=1.6cm of labc]{d)}[];
		    \node (rf3) [node, below=3.1cm of s, xshift=1.6cm]{\textbf{v}}[->]
  		    child {
			node [node, xshift=0.25cm]{\textbf{a}}[]
		    }
		    child {
		      node [node]{\textbf{x}}[]
		      child {
			node [node]{\textbf{d}}[]
			child {
			  node [node]{\textbf{e}}[]
			}
			child {
			  node [node]{\textbf{f}}[]
			}
		      }
		    }
                   child {
		      node (ei) [node, xshift=0cm]{\textbf{e}}[]
		    }
		    child {
			node (eei) [node, xshift=0.2cm]{\textbf{e}}[]
		    };
      \node[right=5.1cm of labc]{e)}[];
      \node[right=7.0cm of labc, yshift=-0.65cm]{$l=2$}[];
		    \node (rf4) [node, right=3.0cm of rf3]{\textbf{v}}[->]
  		    child {
			node [node, xshift=0.25cm]{\textbf{a}}[]
		    }
		    child {
		      node [node]{\textbf{x}}[]
		      child {
			node [node]{\textbf{d}}[]
			child {
			  node [node]{\textbf{e}}[]
			}
			child {
			  node [node]{\textbf{f}}[]
			}
		      }
		    }
                   child {
		      node (ei) [node, xshift=0cm]{\textbf{e}}[]
		      child {
			node [node, xshift=0.1cm]{\textbf{e}}[]
		      }
		      child {
			node [node, xshift=0.1cm]{\textbf{f}}[]
		      }
		    }
		    child {
			node (eei) [node, xshift=0.2cm]{\textbf{e}}[]
			child {
			  node [node]{\textbf{g}}[]
			}
		    };
\end{tikzpicture}
  \caption{Feature space representation related to the kernel $\nuovotk$: a) an input DAG; b) the proper subtree rooted at the node labelled as \textbf{v}; c)-e) given the child \textbf{x} of \textbf{v}, the features related to visits limited to $l$ levels. \label{fig:nuovokernelfeaturespace}}
\end{figure}
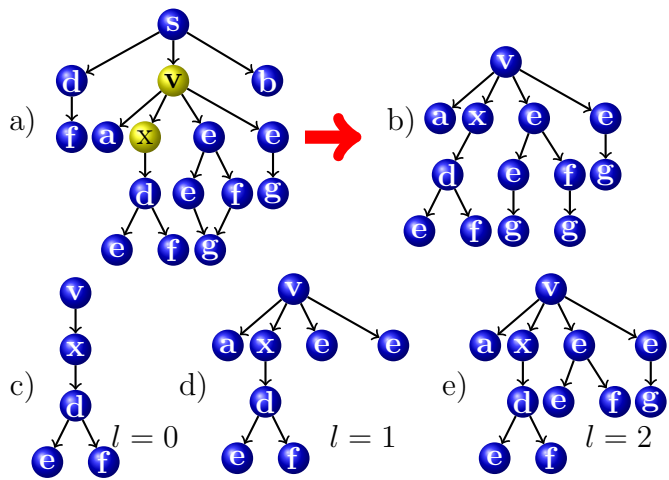 
For each node $\node\in V_D$, for example the node with label {\bf v} highlighted in Figure~\ref{fig:nuovokernelfeaturespace}-a, the algorithm inserts 
the following features:

\begin{enumerate}
 \item the proper subtree rooted at $\node$, which in our example is the one in  Figure~\ref{fig:nuovokernelfeaturespace}-b; 
 \item given $ch_j[\node]$, the subtree composed by: 
  \begin{itemize}
    \item $\node$;
    \item the proper subtree rooted at the $j$-th child of $\node$; 
    \item the subtrees resulting from a visit limited to $1\leq l\leq h$ levels starting from the other children of $\node$ 
  \end{itemize} 
  is added as feature. As $l$ ranges from $0$ to $h$, the features/subtrees from  \mbox{Figure~\ref{fig:nuovokernelfeaturespace}-c} to \mbox{Figure~\ref{fig:nuovokernelfeaturespace}-e} are added. 
\end{enumerate}
%
Recalling that $H$ is the number of nodes in a DAG ODD$_G^\node$, the complexity of Algorithm~\ref{alg:nuovokernelfeatures} is $O(Hh^2\rho^2\log \rho)$. 
The complexity of the ODD kernel in \eqref{eq:k}, instantiated with $\nuovotk$~as base kernel is $O(|V_G|\log |V_G|)$\textcolor{black}{, assuming $\rho$ constant}.
%
%
\subsection{A Novel Feature Weighting Scheme} \label{sec:tanh}
{
The features associated with many kernels for graphs, including $\text{ODD}_{ST_h}$ and ODD$_{\nuovotk}$, are not independent from each other. 
They are, instead, organized in a hierarchical structure~\cite{Yanardag2014}. 
Let us consider the $\text{ODD}_{ST_h}$ kernel as an example: given any pair $t_i, t$ such that $t_i$ is a subtree of $t$, if $t$ occurs as a feature for a graph $G$, then $t_i$ must occur as features as well. 
As a consequence, sticking to our example, there is a monotonic increasing relationship between the frequencies of the subtree features $t_i$ and the subtree features $t$. 
Such relationship is quantified in the upper-left plot of Figure~\ref{fig:featfreq}, which reports the frequencies of the features generated by the $ODD_{ST_h}$ kernel, for $h\in\{0,\ldots,3\}$, on one of the datasets we will consider in Section~\ref{sec:exps}. 
The points in the $x$-axis correspond to features, sorted according to their weights. 
The y-axis, since $\lambda=1$, reports the frequencies of the features in the dataset, i.e. the number of times each feature appears in all input graphs. 
Note that the $x$-axis is in logarithmic scale. 
The frequencies are distributed according to a Zipfian distribution, which means that there are very few features with high frequency. 
Given the structured nature of the feature space, such features are the ``simple'' ones, i.e. those associated with small sized subtrees, for example single nodes. 
Any kernel function evaluation will then be highly influenced by such features, which are typically the least discriminative ones. 
In the case of the $\text{ODD}_{ST_h}$ and ODD$_{\nuovotk}$ kernels, which we recall first decompose the graph into a set of DAGs, the difference between the frequencies of small-sized and large-sized features is even greater since they are extracted from multiple DAGs: the smaller the size of a subtree, the more likely for it to appear in multiple DAGs. 
The fact that the distribution of weights of the features is particularly skewed, may negatively impact the predictive performance of the kernel since, in principle, we would like to give more emphasis to (i.e. to weight more) bigger, discriminative features with respect to small ones, that tend to appear in almost all examples, and thus are generally not correlated with the target concept. 
}
%
%
\begin{figure}[t]
{
	\centering
	\includegraphics[width=1.0\columnwidth]{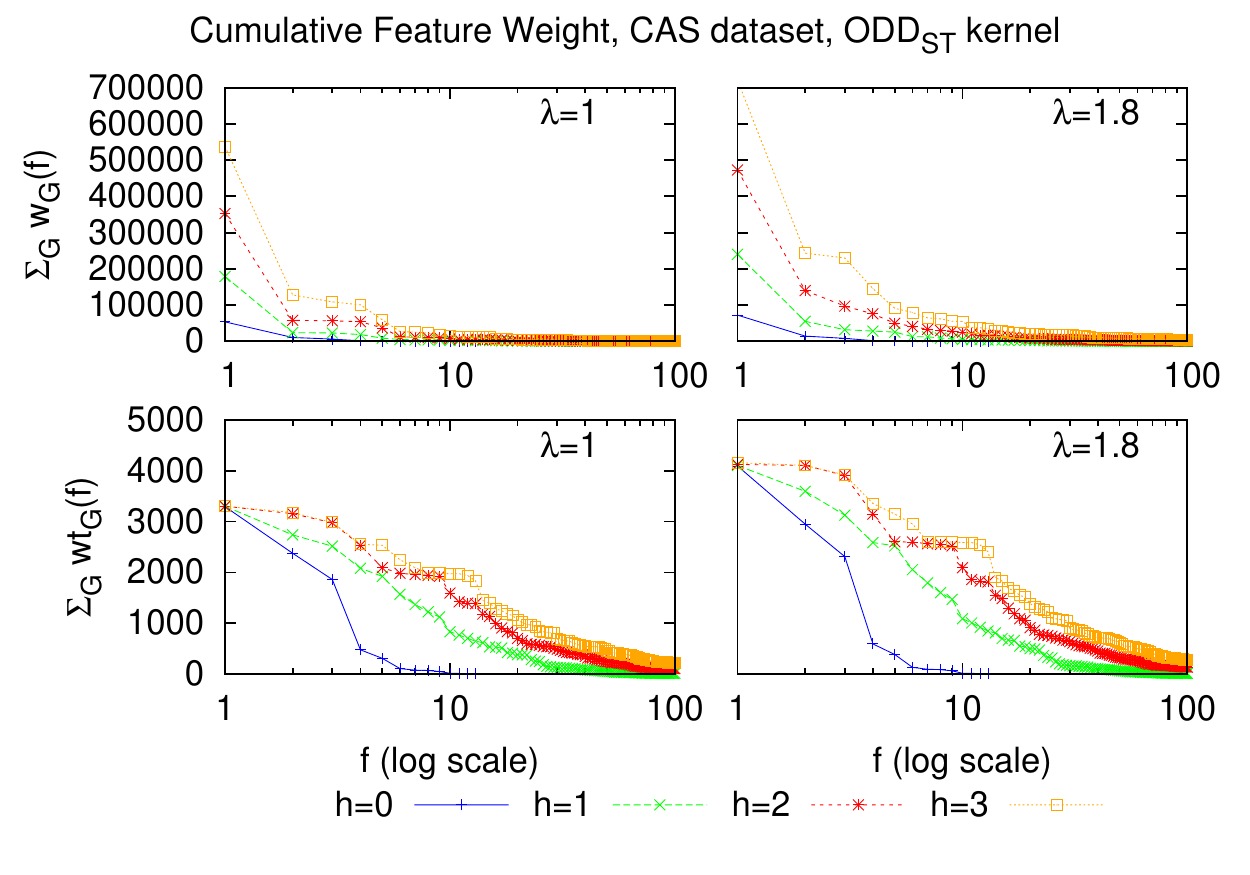}
	\caption{Comparison between the weighting schemes $w_G(f)$~\eqref{eq:lambda} and $wt_G(f)$~\eqref{eq:tanh}. On the $x$ axis, in a logarithmic scale, the first 100 features generated by the ODD$_{ST_h}$ kernel for different $h$ values.
	The $y$ axis reports the cumulative weight of each feature among all the graphs in the dataset.\label{fig:featfreq}}
}
\end{figure}

{One way to tackle this issue is to adopt the weighting scheme explained in Section~\ref{sec:st}, that has been designed specifically for the case of the computation of tree kernels~\cite{Collins2001}.}
{This scheme has been implemented in the original $ODD_{ST_h}$ kernel formulation, and we maintained it for the proposed $ODD_\nuovotk$ kernel: given  a graph $G$, the weight $w_G(f)$ of each feature $f$ (see lines 4 and 9 of Algorithm~\ref{alg:nuovokernelfeatures}) is computed as }
 \begin{equation}
w_G(f)=freq_G(f) \cdot \lambda^\frac{|f|}{2},
\label{eq:lambda}
 \end{equation}
  where $freq_G(f)$ is the frequency of the feature $f$ in $G$. 
  Therefore the contribution to the kernel of the same matching feature (computed via dot product) in two input graphs $G_1$ and $G_2$ is $freq_{G_1}(f)\cdot freq_{G_2}(f) \cdot \lambda^{|f|}$. 
{A value of $\lambda$ greater than $1$ would give more importance to large matching trees. 
However, the contribution of the less frequent, possibly interesting, small features could be underweighted. } 
{The upper-right plot in Figure~\ref{fig:featfreq} shows that, with this weighting approach, there are slightly more features with a relatively high weight w.r.t. the case where no weighting scheme is applied (i.e. when $\lambda=1$). Nonetheless, the distribution is still very skewed. }\\
Another possibility is to define a different weighting scheme, more suited to our approach.
\textcolor{black}{As a first step in this direction,} we propose to mitigate the contribution of \textcolor{black}{otherwise} overweighted features 
with a different definition\footnote{This is an evolution of the scheme proposed in~\cite{DaSanMartino2015}.} of $w_G(f)$,
{in the following denoted as $wt_G(f)$, i.e.
\begin{equation}
    wt_G(f)=\tanh(\lambda^{|f|})\cdot \tanh(freq_G(f)), 
   \label{eq:tanh}
\end{equation}
\noindent where $\tanh(\cdot)$ is the hyperbolic tangent function.}
%
\textcolor{black}{Note that the original weighting scheme depends nonlinearly (exponentially) on the size of the feature $|f|$ and linearly on its frequency.
The novel scheme we are proposing, on the other hand, depends nonlinearly on both $|f|$ and $freq_G(f)$.
In this way, the contribution of each feature is smoothly and non-linearly normalized in the interval $[0,1]$.}\\
Note that the hyperbolic tangent function is almost linear around zero, and asymptotically tends to one for positive values. 
This means that the contribution of frequent features is truncated, while the less frequent features are still discriminated since they fall in the linear part of the function. 
{The same is true for the $\lambda^{|f|}$ factor.}

{The lower plots in Figure~\ref{fig:featfreq} reports the weights distribution according to the new $wt_G$ weighting function proposed in~\eqref{eq:tanh} with $\lambda=1$ and $\lambda=1.8$, respectively. The final result is that the weights are distributed in a smoother way.}\\
The new weighting scheme is applied to the ST kernel, obtaining a variant of the kernel proposed in \cite{Dasan2012}, and to the ST+ kernel first proposed in this paper. 
{Note that this novel weighting scheme is just one possibility among several ones. The key point is that we want to achieve a smoother distribution of the weights associated to the features.}
{The tanh function implements all our desiderata, but any other sigmoidal function can be adopted. Notwithstanding the heuristic
nature of our choice, the experimental results we have obtained on several real world datasets, as reported in Section~\ref{sec:exps}, show
that the novel proposed weighting scheme allows to reach statistically significant improvements over state-of-the-art kernels. This seems to
confirm that both our intuition on the smoothness of the weight distribution, as well as its implementation via the tanh function, are useful.} 

\section{Related work} \label{sec:relatedwork}

Graph data is usually high-dimensional. For this reason, in order to perform learning on graph datasets, there are two possible approaches: 
\begin{enumerate}
  \item applying a preprocessing phase aimed at selecting possibly relevant features; 
  \item in the context of kernel methods, using tractable kernel functions. 
\end{enumerate}

Generally speaking, the methods following the first approach extract frequent patterns, build a vectorial representation of the graphs according to such patterns and then apply a kernel method. When the kernel method is an SVM, the approach is referred to as SVM with frequent pattern mining (freqSVM). The techniques for extracting the features include Gaston~\citep{Kazius2006}, Correlated Pattern Mining (CPM)~\citep{DBLP:conf/pkdd/BringmannZRN06}, MOLFEA~\citep{Helma2004}. 
Saigo et al.~\citep{Saigo2009} proposed gBoost, a 
boosting method that progressively collects informative (according to the target output) patterns.
 
The second approach includes a set of kernel functions for graphs. 
The Marginalized Graph Kernel (MGK) considers common walks as features~\citep{Kashima2003} 
(the work has been extended in order to make it more efficient and effective in~\citep{Mah'e2004}). 
Informally, this kernel is defined as the expected value of a kernel over all possible pairs of label sequences generated by random walks on two graphs. The worst case time complexity of the algorithm presented in~\citep{Vishwanathan2006} is $O(|V_G|^3)$. 

The Shortest Path Kernel associates a feature to each pair of nodes of one graph. The value of the feature is the length of the shortest path between the corresponding nodes in the graph~\citep{Borgwardt2005}. The complexity of the kernel is $O(|V|^4)$. 
Being the Shortest Path Kernel based on paths, it can be represented as an instance of \eqref{eq:k}. 
{We do not report experimental results about this kernel because of its high computational complexity, and its inferior results compared to other state of the art kernels on many of the datasets considered in this paper~\citep{Shervashidze2009a,Shervashidze2011}.}

In \citep{Heinonen2009} it is described an effective method for computing path based kernels. First a graph is decomposed into a set of trees of totally $t$ nodes. Then the Burrows-Wheeler transform is employed for fast and space-efficient enumeration of paths. The complexity of the kernel is $O(t\log t^\epsilon)$, with $\epsilon<1$. 
{The \textit{graphlet} kernel~\citep{Shervashidze2009}  counts all types of matching subgraphs of small size $k$ (e.g. $k=3,4$ or $5$). There are efficient schemes for computing this kernel, but they are applicable only on unlabeled graphs. For the labeled case, the computational complexity of this kernel is $O(n^k)$. {In the experimental section of this paper, we considered the Graphlet kernel instantiated with $k=3$, that will be referred as {3-Graphlet} kernel.}}

{The Weisfeiler-Lehman Fast Subtree kernel (FS) counts the number of identical subtree patterns obtained by subtree-walks up to height $h$~\citep{Shervashidze2009a,Shervashidze2011}. The complexity of the kernel is $O(|E|h)$. 
While being fast to compute, the kernel may lack of expressiveness for some tasks given that the number of non-zero features generated by one graph is at most $|V|h$. 
Note that the subtree-walks extracted by the kernel differ from the tree structures extracted by the kernels proposed in Section~\ref{sec:dagkernels}: in FS a node usually appears multiple times in the same subtree-walk, while in the ODD kernel only DAG nodes which have multiple incoming edges appear multiple times in the extracted tree structures. 
Such difference makes the Weisfeiler-Lehman Fast Subtree kernel not reproducible from \eqref{eq:k}; a discussion on the differences between the feature spaces induced by the Weisfeiler-Lehman Fast Subtree and the $ODD_{ST_h}$ kernels can be found in \cite{Dasan2012}. 
Moreover, the features of the FS kernel are  subtree-walks, while specific features (as explained in Sections~\ref{sec:st} and~\ref{sec:noveltk}) are extracted from the tree-visits obtained from the $ODD_{ST_h}$ and $ODD_{ST+}$ kernels. 
}

Costa and De Grave \citep{Costa2010} extended the Fast Subtree Kernel by computing exact matches between pairs of subgraphs with controlled size and distance. 
{Their kernel, named Neighborhood Subgraph Pairwise Distance Kernel (NSPDK)}, has $O(|V| |V_h| |E_h| \log |E_h|)$ time complexity, where $|V_h|$ and $|E_h|$ are the number of nodes and the number of edges of the subgraph obtained by a breadth-fist visit of depth $h$. 
The authors state that, for small values of the subgraph size and distance, the complexity of the kernel becomes in practice linear.

{The {Weisfeiler-Lehman} Shortest path Kernel proposed in \citep{Shervashidze2011} is similar in spirit to the NSPDK kernel. Indeed, it considers pairs of subtree patterns and their distance. However it does not limit the maximum distance between the considered patterns, resulting in a computational complexity of $O(n^4)$.
}
Mah\'e and Vert \citep{Mah'e2009} described a graph kernel based on extracting tree patterns from the graph. The difference with the approach of this paper is that the tree patterns are obtained as result of walks on the graph, i.e. the same node can appear more than once in the same tree pattern. The complexity of the kernel is $O(|V_1| |V_2| h\rho^{2\rho} )$, where $h$ is the depth of the visit.  
{Finally, \citep{Gauzere2015} proposed the treelet kernel, based on frequent pattern mining of tree-substructures. The kernel implementation considers subtrees with a maximum of $6$ nodes, and its computational complexity is $O(n\rho^5)$. }
\begin{table}[t]
\centering
{
  \begin{tabular}{|l|c|}
  \hline
  Kernel & Complexity  \\
  \hline
  RW~\citep{Kashima2003}  & $O(|V|^3)$ \\
  SP~\citep{Borgwardt2005} & $O(|V|^4)$  \\
  WL-SP~\citep{Shervashidze2011} & $O(|V|^4)$  \\
  3-Graphlet~\citep{Shervashidze2009} & $O(|V|^3)$  \\
  Treelet~\citep{Gauzere2015} & $O(|V|\rho^5)$ \\
  FS~\citep{Shervashidze2009a,Shervashidze2011} & $O(|E|h)^{*}$  \\
  NSPDK~\citep{Costa2010}  & $O(|V|)^{*,**}$  \\
  \hline
  ODD$_{ST}$~\cite{Dasan2012} & $ O(|V| log |V|)^* $ \\
   ODD$_{ST+}$& $O(|V| log |V|)^*$\\
    \hline
  \end{tabular}
  \caption{Computational complexity of the Shortest Path, the 3-Graphlet, the fast Subtree, the NSPDK, the ODD$_{ST}$ and ODD$_{ST+}$ kernels. *: considering $\rho$ constant; **: with high constants.\label{tab:kernelssummary}}
 }
\end{table}
\begin{figure}[ht]
{
	\centering
	\includegraphics[width=0.85\columnwidth]{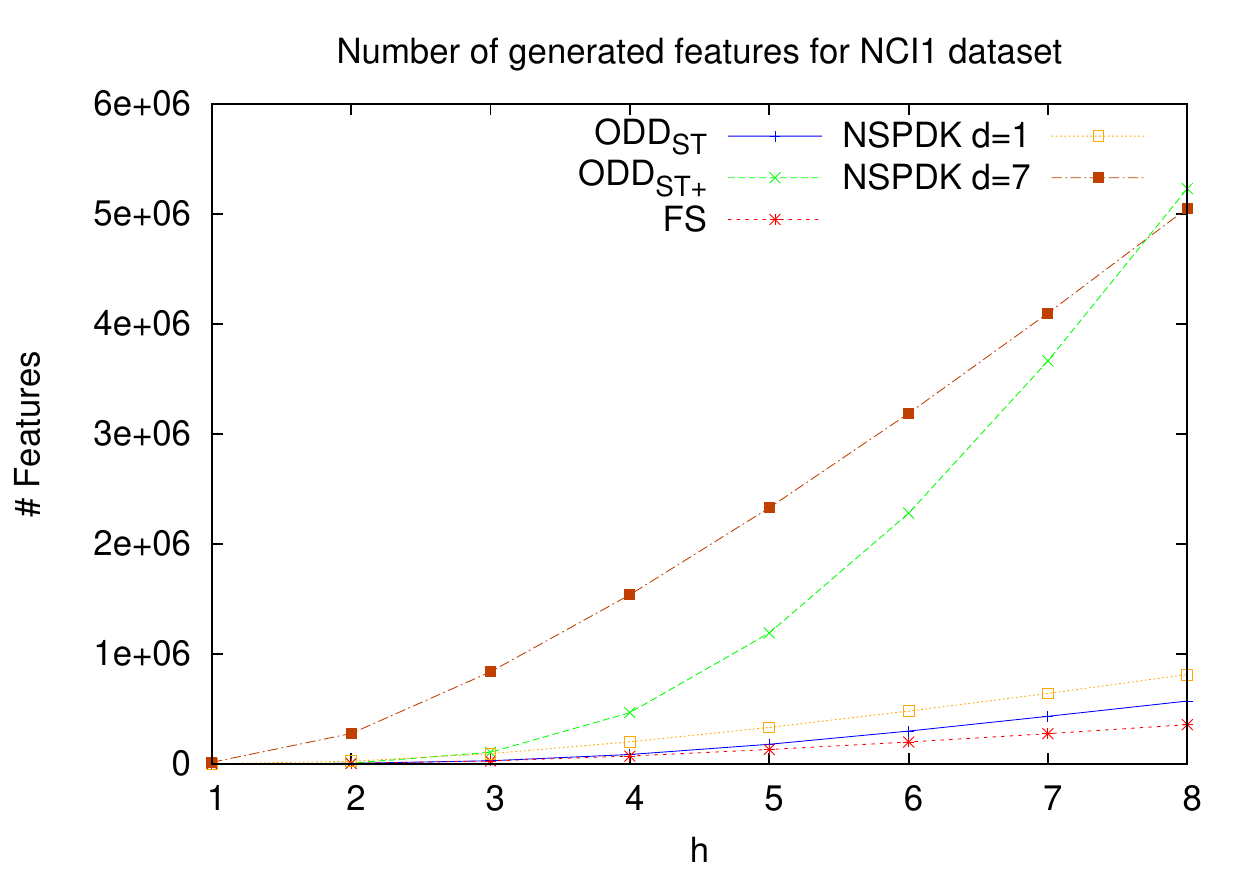}
	\caption{{Number of features generated by the $ODD_{ST_{h}}$, $ODD_{ST_{+}}$, FS and NSPDK kernels on the NCI1 dataset as a function of their parameter $h$.}\label{fig:nfeatures}}
}
\end{figure}
{Table~\ref{tab:kernelssummary} summarizes the computational complexity of some of the kernels cited in this section, and the ones proposed in this paper.
}
{Moreover, just to give an idea about how many features are generated by a graph kernel on a real-world dataset, in Figure~\ref{fig:nfeatures} we have reported the number of different features generated on a chemical dataset (NCI1) by the most efficient aforementioned kernels.}
\section{Experimental results}\label{sec:exps}
\subsection{{Experiments on common benchmark graph datasets}}\label{expsold}
\begin{table}[t]
\center
\begin{tabular}{|l|c|c|c|c|}
\hline
Dataset	&	graphs	& 	pos(\%)	&	avg nodes	& avg edges \\
\hline
CAS	&	$4337$	&	$55.36$	&	$29.9$		& $30.9$	\\
CPDB	&	$684$	&	$49.85$	&	$14.1$		& $14.6$	\\
AIDS	&	$1503$	&	$28.07$	&	$58.9$		& $61.4$	\\
NCI1	&	$4110$	&	$50.04$	&	$29.9$		& $32.3$	\\
NCI109 & $4127$ & $50.37$ & $29.7$ & $32.1$\\
GDD	&	$1178$	&	$58.65$	&	$284.3$	& $2862.6$	\\		
\hline
MSRC\_9 & 221 & multi-class & 40.6 & 97.9 \\
MSRC\_21& 563& multi-class & 77.5 & 198.3 \\
\hline
NCI123 & 40952 & 4.76 & 26.8 & 28.9 \\
NCI\_AIDS & 42682 & 3.52 & 45.7 & 47.7 \\
\hline
\end{tabular}
\caption{Statistics of CAS, CPDB, AIDS, NCI1 , NCI109, GDD, MSRC\_9, MSRC\_21, NCI123 and NCI\_AIDS datasets: number of graphs, percentage of positive examples, average number of atoms, average number of edges.\label{tab:datasets}}
\end{table}
The experimental assessment of the proposed kernels has been performed on a total of eight datasets.  The first six datasets involve chemo and bioinformatics data: CAS\footnote{http://www.cheminformatics.org/datasets/bursi},
CPDB~\cite{Helma2004}, AIDS~\cite{Weislow1989}, NCI1, NCI109~\cite{springerlink:10.1007/s10115-007-0103-5} and GDD~\cite{dobson2003}. 
The first five datasets involve chemical compounds and represent binary classification problems. The nodes are labeled according to the atom type and the edges represent the bonds. 
GDD is a dataset composed by proteins represented as graphs, where the nodes of the graphs represent amino acids and two nodes are connected by an edge if they are less than 6{\AA} apart. 
\textcolor{black}{Moreover, we adopted from \cite{Birlinghoven} two real-world image datasets: MSRC9-class and MSRC21-class\footnote{\mbox{http://research.microsoft.com/en-us/projects/ObjectClassRecognition/}}. Each image is represented by its conditional Markov random field graph enriched with semantic labels, and the task is scene classification. Both the datasets are multi-class single-label classification problems.
For our experiments, we adopted a SVM classifier~\cite{Pedregosa2012}. For the multi-class problems, we adopted a  one-vs-one scheme.}
\begin{table}[p]
\begin{center}
\begin{tabular}{|l|c|c|c|c|}
\hline
\scriptsize $Kernel$ 		& \scriptsize  CAS 		& \scriptsize CPDB 		& \scriptsize AIDS~~ 		& \scriptsize NCI1~~  \\ 
\hline
\scriptsize $p$-random walk & 70.16{\bfseries *} (8) \vspace{-0.1cm}& 64.14{\bfseries *} (8) & 73.55{\bfseries *} (8)  & -  
\vspace{-0.1cm}\\
 &  {\tiny $\pm 0.20$}& {\tiny $\pm 1.35$} & {\tiny$\pm 0.49$ }& {\tiny$\pm -$} \\
 \scriptsize Graphlet & 71.10{\bfseries *} (7) \vspace{-0.1cm}& 67.36{\bfseries *} (7) & 73.98{\bfseries *} (7)  & 69.68{\bfseries *} (7) 
\vspace{-0.1cm}\\
 &  {\tiny $\pm 0.48$}& {\tiny $\pm 0.96$} & {\tiny$\pm 0.65$ }& {\tiny$\pm 0.52$} \\
\scriptsize FS 	&  83.32{\bfseries *} (6)	&  76.36 (5) & 82.02~~(5)  &  84.41 (4) \vspace{-0.1cm}\\ 
 & {\tiny $\pm 0.37$} & {\tiny $\pm 1.48$} & {\tiny $\pm 0.4$} & {\tiny $\pm 0.49$} \\
\scriptsize NSPDK			& 83.60{\bfseries *} (2)	& {\bfseries 76.99}	(1)  	& {\bfseries 82.71} (1) 	&83.45  (5) \vspace{-0.1cm}\\ 
& {\tiny $\pm 0.34$} &  {\tiny $\pm 1.15$}&   {\tiny $\pm 0.66$} & {\tiny $\pm 0.43$} \\
\scriptsize ODD$_{\text{ST}_h}$		& 83.34{\bfseries *} (4)	& 76.44  (4)  	& 81.51 (6) 	& 82.10{\bfseries *} (6)\vspace{-0.1cm}\\ 
& {\tiny $\pm 0.31$}  &  {\tiny $\pm 0.62$} & {\tiny $\pm 0.74$} & {\tiny $\pm 0.42$}  \\
\scriptsize ODD$_{\text{ST}_h}^{\text{\tiny{TANH}}}$		& {83.40}{\bfseries *} (3)	& 76.56 (3)  	& 82.51  (3) 	& 84.57  (3)	 \vspace{-0.1cm}\\
& {\tiny $\pm 0.41$} & {\tiny $\pm 0.97$} & {\tiny $\pm 0.52$} & {\tiny $\pm 0.43$} \\
\scriptsize ODD$_{\text{ST}_+}$		&  {\bfseries 83.90} (1)	& 76.30  (6)  	& 82.06  (4) 	& {\bfseries 84.97} (1)\vspace{-0.1cm}\\
& {\tiny $\pm 0.33$} & {\tiny $\pm 0.23$} & {\tiny $\pm 0.70$} & {\tiny $\pm 0.47$}  \\
\scriptsize ODD$_{\text{ST}_+}^{\text{\tiny{TANH}}}$		& {83.33}{\bfseries *} (5)	& 76.74 (2)  	& 82.54 (2) 	& 84.81  (2)	\vspace{-0.1cm}\\
& {\tiny $\pm 0.34$} & {\tiny $\pm 1.81$} & {\tiny $\pm 0.75$} & {\tiny $\pm 0.41$}\\
\hline
\hline
\scriptsize $Kernel$  	& \scriptsize GDD & \scriptsize NCI109 & \scriptsize  MSRC\_9 		& \scriptsize MSRC\_21 \\ 
\hline
\scriptsize $p$-random walk & - \vspace{-0.1cm}& - & 67.01{\bfseries *} (7)  & 18.88{\bfseries *} (8) 
\vspace{-0.1cm}\\
 &  {\tiny $\pm -$}& {\tiny $\pm -$} & {\tiny$\pm 2.22$ }& {\tiny$\pm 1.4$} \\
\scriptsize 3-Graphlet & 74.92 (6) \vspace{-0.1cm}& 68.07{\bfseries *} (7) & 60.83{\bfseries *} (8)  & 19.66{\bfseries *} (7) 
\vspace{-0.1cm}\\
 &  {\tiny $\pm 1.40$}& {\tiny $\pm 0.31$} & {\tiny$\pm 2.0$ }& {\tiny$\pm 0.96$} \\
\scriptsize FS & 75.46 (3) \vspace{-0.1cm}& {\bfseries85.02} (1) & 89.26{\bfseries *} (6)  & 89.87 (6) 
\vspace{-0.1cm}\\
 &  {\tiny $\pm 0.98$}& {\tiny $\pm 0.44$} & {\tiny$\pm 0.82$ }& {\tiny$\pm 0.71$} \\
\scriptsize NSPDK		& $74.09$  (7) \vspace{-0.1cm}& $84.17$ (2) & $89.48${\bfseries *} (4) & $90.24$ (3) \vspace{-0.1cm}\\ 
& {\tiny $\pm 0.91$} & {\tiny $\pm 0.33$} & {\tiny$\pm 1.0$ }& {\tiny$\pm 0.49$}\\
\scriptsize ODD$_{\text{ST}_h}$	& $75.27$  (5) & $81.91${\bfseries *} (6) & $90.80$ (3)  & $89.92$ (5) \vspace{-0.1cm}\\ 
& {\tiny $\pm 0.68$} & {\tiny $\pm 0.42$} & {\tiny $\pm 1.10$} & {\tiny$\pm 0.73$} \\
\scriptsize ODD$_{\text{ST}_h}^{\text{\tiny{TANH}}}$	& {\bfseries 76.09}  (1) \vspace{-0.1cm}& 83.68 (4)  & \textbf{94.39} (1)  & \textbf{92.60 }(1) \vspace{-0.1cm}\\
 &{\tiny $\pm 0.85$}& {\tiny $\pm 0.39$} & {\tiny$\pm 1.21$} & {\tiny$\pm 0.45$}\\
\scriptsize ODD$_{\text{ST}_+}$	& $75.33$  (4) \vspace{-0.1cm}& $83.08${\bfseries *} (5)  & 89.33{\bfseries *} (5) & 89.94 (4) \vspace{-0.1cm}\\
 & {\tiny $\pm 0.81$}& {\tiny $\pm 0.49$}& {\tiny$\pm 1.2$} & {\tiny$\pm 0.80$ } \\
\scriptsize ODD$_{\text{ST}_+}^{\text{\tiny{TANH}}}$		& $75.52$ (2) \vspace{-0.1cm}& 83.93 (3) & 92.99 (2)  & 91.74 (2) \vspace{-0.1cm}\\
& {\tiny $\pm 0.88$} & {\tiny $\pm 0.42$} & {\tiny$\pm 1.26$} & {\tiny$\pm 0.77$}\\
\hline
\end{tabular}
\end{center}
 \caption{Average accuracy results $\pm$ standard deviation in nested 10-fold cross validation for the $p$-random walk, the Graphlet, the Fast Subtree, the Neighborhood Subgraph Pairwise Distance, the $ODD_{ST_h}$, the  ODD$_{\text{ST}_h}^{\text{\tiny{TANH}}}$, the $ODD_{\nuovotk}$ and the ODD$_{\text{ST}_+}^{\text{\tiny{TANH}}}$ kernels on CAS, CPDB, AIDS, NCI1, GDD, NCI109, MSRC\_9 and MSRC\_21  datasets. The rank of the kernel is reported between brackets. {The symbol {\bfseries *} denotes the kernels whose performance difference with respect to the top-ranked kernel is statistically significant}.  \label{tab:nestedkfoldresults}}
\end{table}
We compare the predictive abilities of the $ODD_{ST_{+}}$ kernel and the two proposed variants $ODD_{\text{ST}_h}^{\text{\tiny{TANH}}}$ and $ODD_{\text{ST}_+}^{\text{\tiny{TANH}}}$ to the original $ODD_{ST_h}$ kernel~\cite{Dasan2012}, the Fast Subtree Kernel (FS)~\cite{Shervashidze2009a} and the Neighborhood Subgraph Pairwise Distance Kernel (NSPDK)~\cite{Costa2010}. 
{Moreover, we also report the performances of the $p$-random walk kernel, that is a kernel that compares random walks up to length $p$ in two graphs (special case of~\citep{Kashima2003} and~\citep{Mah'e2004} ) as representative for the family of kernels based on random walks, and the \textit{graphlet} kernel~\citep{Shervashidze2009}.
Note that the complexity of the \textit{graphlet} kernel (when applied to labeled graphs) is exponential in the size $k$ of the \textit{graphlet}. Because of that, following \cite{Shervashidze2011}, we restricted our experimentation to a value of $k$ that allows for an efficient computation of the kernel, i.e. $k=3$. 
}

The experiments are performed using a \textit{nested} 10-fold cross validation: for each of the 10 folds another \textit{inner} 10-fold cross validation, in which we select the best parameters for that particular fold, is performed. 
All the experiments have been repeated $10$ times using different splits for the cross validation, and the average results (with standard deviation) are reported. 
For all the experiments, the values of the parameters of the $ODD_{ST_h}$ and $ODD_{\nuovotk}$~kernels,  { including their variants using \textit{tanh}}, have been restricted to: $\lambda=\{0.1, 0.2, \ldots, 2.0\}$, $h=\{1, 2, \ldots, 10\}$. 
For the Fast Subtree kernel the only parameter $h=\{1, 2, \ldots, 10\}$ is optimized.
For the NSPDK, the parameters $h=\{1, 2, \ldots, 8\}$ and $d=\{1, 2, \ldots, 7\}$ are optimized. 
{Finally, for the $p$-random walk kernel we selected $p=\{1, 2, \ldots, 10\}$, and for the \textit{graphlet} kernel we considered only the graphlets of size $3$, as mentioned above.} 
{A 10x10 CV test with confidence level 95\% (and 10 degrees of freedom) has been executed between each pair of kernels on all datasets \cite{Japkowicz:2011:ELA:1964882}. In the following the term significant will refer to this statistical test.}  
Table~\ref{tab:nestedkfoldresults} reports the average accuracies and the rankings obtained by the different kernels on the considered datasets.
{The symbol {\bfseries *} in Table~\ref{tab:nestedkfoldresults} denotes, for each dataset, the kernels whose performance difference with respect to the top-ranked kernel is statistically significant.} 

Let us now focus on the experimental results obtained for the six chemical datasets.
The kernels $ODD_{\text{ST}_h}^{\text{\tiny{TANH}}}$,  $ODD_{\nuovotk}$, $ODD_{\text{ST}_+}^{\text{\tiny{TANH}}}$ together have best accuracy on three out of six datasets, and the second best accuracy on two others. 
{On the datasets in which the FS and NSPDK kernels perform better than the ODD ones, i.e. CPDB, AIDS and NCI109, the performance difference, at least with respect to the best performing ODD kernel, is never significant. 
Note that $ODD_{\nuovotk}$ performs significantly better than NSPDK and FS on the CAS dataset.} 
The variant employing the hyperbolic tangent is always useful for the ST kernel, making it the best performing kernel on GDD, and is able to boost the accuracy performance of $ODD_{\nuovotk}$ on AIDS, CPDB , GDD and NCI109 datasets. 
The generally good results of the ODD kernels, with respect to FS and NSPDK, may be attributed to the fact that they have associated a large feature space, which makes them more adaptable to different tasks. 
{Note that the execution of $p$-random walk kernel did not complete in 4 days for NCI1, NCI109 and GDD datasets, so the results are missing.}

\textcolor{black}{Let us now focus on the image datasets (MSRC\_9 and MSRC\_21).
On these datasets, the baselines FS, NSPDK, $ODD_{\text{ST}_h}$ kernels and the proposed $ODD_{\nuovotk}$ kernel show very similar performances.
On these datasets, the introduction of the hyperbolic tangent weighting scheme is very beneficial. Both $ODD_{\text{ST}_h}^{\text{\tiny{TANH}}}$ and $ODD_{\text{ST}_+}^{\text{\tiny{TANH}}}$ performs better than all the baselines, with the former being the best performing kernel on both datasets. }\\
{The $p$-random walk kernel and the \textit{graphlet} kernel show poor performances on these datasets. We argue that this is because they are the only ones among the considered kernels that do not consider all the neighbors of a node as a feature.
}
\begin{figure}
	\centering
	\includegraphics[width=0.8\columnwidth]{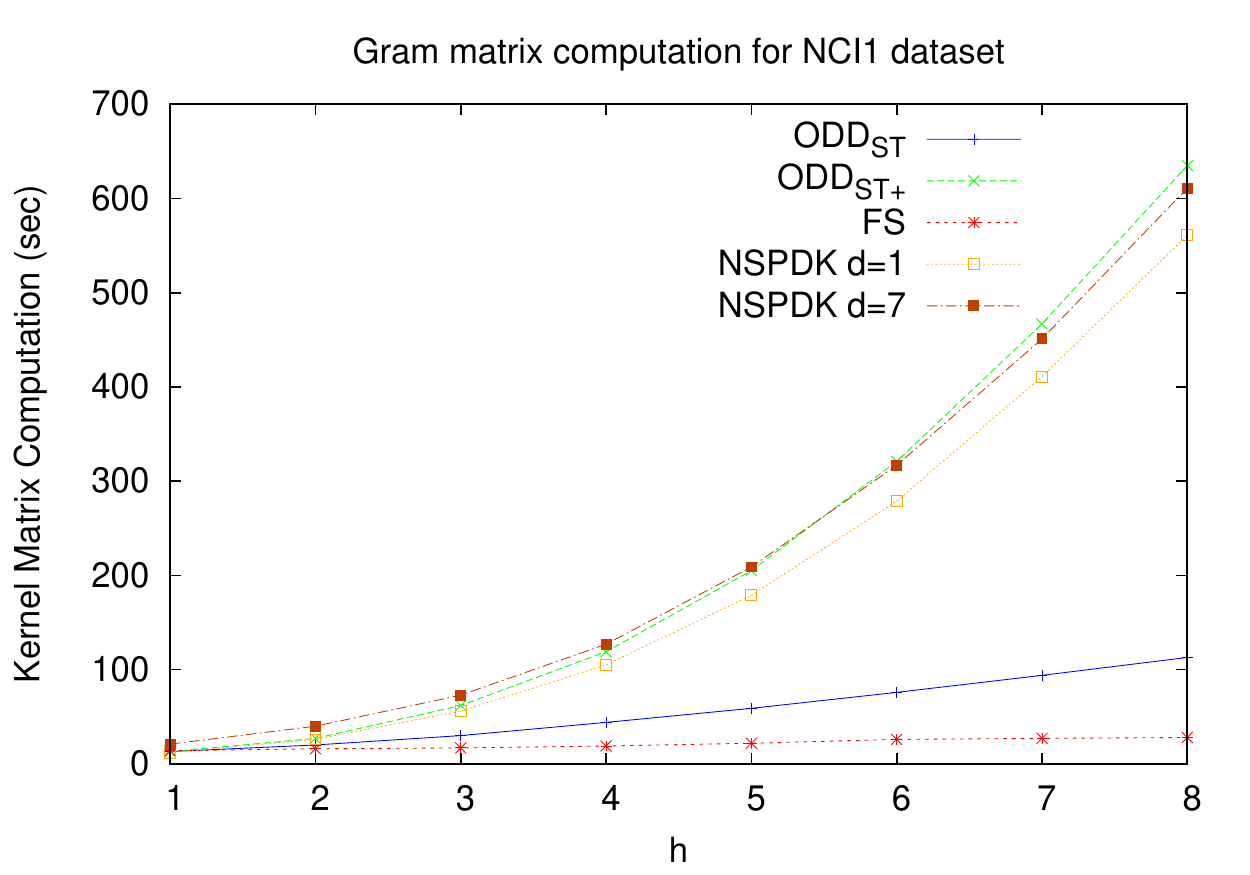}
\caption{{Time needed to compute the kernel matrix for the $\text{ODD-ST}_{h}$, $\text{ODD-\nuovotk}_{h}$, the NSPDK and the FS kernels, as a function of their parameter $h$, on NCI1.\label{fig:tempinci1}}}
\end{figure}
\begin{figure}
	\centering
	\includegraphics[width=0.8\columnwidth]{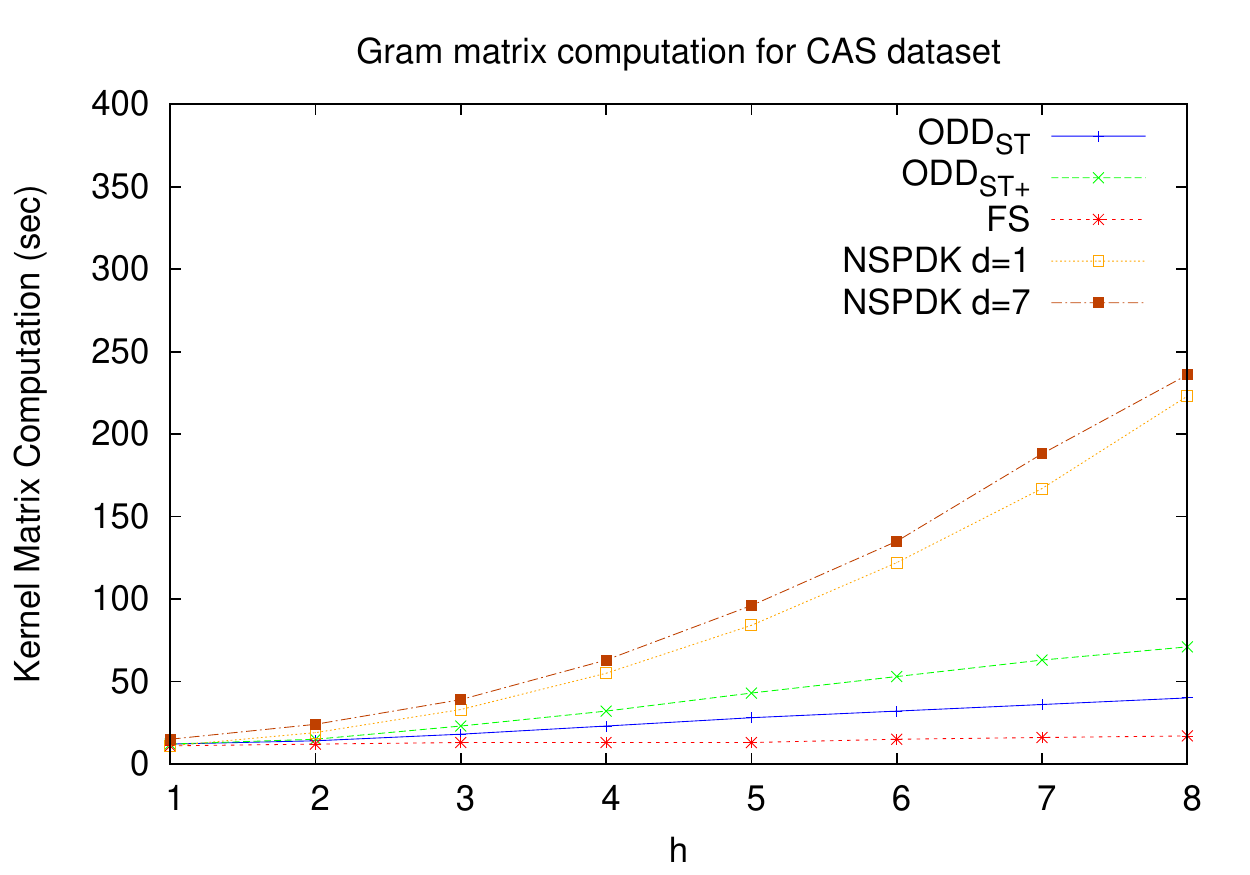}
\caption{{Time needed to compute the kernel matrix for the $\text{ODD-ST}_{h}$, $\text{ODD-\nuovotk}_{h}$, the NSPDK and the FS kernels, as a function of their parameter $h$, on CAS dataset.\label{fig:tempicas}}}
\end{figure}

\textcolor{black}{Figures~\ref{fig:tempinci1} and \ref{fig:tempicas} report the computational times required by the ODD$_{ST_h}$, ODD$_{\nuovotk}$, NSPDK and the FS kernels  as a function of the parameter $h$ determining the size of the considered substructures on the NCI1 and CAS datasets, respectively.}\\

All the experiments are performed on a PC with two Quad-Core AMD Opteron(tm) 2378 Processors and 64GB of RAM. 
{The proposed kernels have been implemented in C++. 
In addition, we implemented a fast version of the FS kernel in C++. 
All these kernels adopt an hashing function, similar in spirit to~\cite{Kersting2013}.
As for the $p$-random walk and \textit{graphlet} kernels, we adopted a publicly available Matlab implementation\footnote{http://www.di.ens.fr/$\sim$shervashidze/code.html}.
Thus, the times for {the $p$-random walk and the \textit{graphlet} kernels} are reported just for a qualitative comparison.\\}
The time needed to compute the kernel matrix for the ODD$_\nuovotk$ kernel increases roughly linearly with respect to the parameter $h$ for both datasets. 
As expected the constant factors are higher than the ones of the ODD$_{ST_h}$, but the ODD$_\nuovotk$ is faster than {(or comparable to)} NSPDK.
\textcolor{black}{Note that we do not report the computational times for ODD$_{ST_h}^{\text{\tiny{TANH}}}$ and ODD$_\nuovotk^{\text{\tiny{TANH}}}$ since their computational requirements are basically the same as the corresponding base kernels: the computation of the novel weight function does not add a significant computational burden.\\} 
\begin{table}[t]
  \centering
 \begin{tabular}{|l|c|c|c|c|}
  \hline
  Kernel	& CAS 			&  AIDS 		& NCI1			& GDD	\\
  \hline
  {Graphlet}		& $58''$ 			&  $54''$  		& $133''$  		& $1715''$	 \\
  {\scriptsize $p$-random walk} & $76h$ 			&  $35h$  		& $-$  		& $-$	 \\
  		& \scriptsize{(h=$7$)}	& \scriptsize{(h=$8$)}	& \scriptsize{(h=$-$)}	& \scriptsize{(h=$-$)} \\
    FS		& $13''$ 			&  $5''$  		& $28''$  		& $17''$	 \\
		& \scriptsize{(h=$3$)}	& \scriptsize{(h=$9$)}	& \scriptsize{(h=$8$)}	& \scriptsize{(h=$1$)} \\
  NSPDK		& $24''$ 			&  $217''$ 		&$192''$ 	 		& $395''$	\\
		& \scriptsize{(h=$2$, d=$6$)}& \scriptsize{(h=$8$,d=$6$)} & \scriptsize{(h=$5$,d=$4$)} & \scriptsize{(h=$2$,d=$6$)} \\
  ODD$_{ST_h}$ 	& $18''$ 			&  $56''$  		& $44''$  		& $29''$  \\
		& \scriptsize{(h=$3$)}	& \scriptsize{(h=$7$)}	& \scriptsize{(h=$4$)}	& \scriptsize{(h=$1$)} \\
		
		 { ODD$_{ST_h}^{\text{\tiny{TANH}}}$ }	& $47''$ 			&  $51''$  		& $110''$  		& $246''$  \\
		& \scriptsize{(h=$5$)}	& \scriptsize{(h=$6$)}	& \scriptsize{(h=$6$)}	& \scriptsize{(h=$2$)} \\
  ODD$_{\nuovotk}$	& $32''$ 			&  $111''$ 		& $205''$ 		& $199''$ \\
		& \scriptsize{(h=$4$)}	& \scriptsize{(h=$8$)}	& \scriptsize{(h=$1$)}	& \scriptsize{(h=$1$)} \\
		{ODD$_{\nuovotk}^{\text{\tiny{TANH}}}$ }	& $179''$ 			&  $61''$  		& $165''$  		& $541''$  \\
		& \scriptsize{(h=$8$)}	& \scriptsize{(h=$5$)}	& \scriptsize{(h=$4$)}	& \scriptsize{(h=$2$)} \\
  \hline
 \end{tabular}
    \caption{Average time required for computing the kernel matrix for the $p$-random walk, the Graphlet, the Fast Subtree, the Neighborhood Subgraph Pairwise Distance, the $ODD_{ST_h}$, the  ODD$_{\text{ST}_h}^{\text{\tiny{TANH}}}$, the $ODD_{\nuovotk}$ and the ODD$_{\text{ST}_+}^{\text{\tiny{TANH}}}$ kernels on CAS, AIDS, NCI1 and GDD datasets with the optimal kernel parameters (reported between brackets). \label{tab:tempi}}
\end{table}

Moreover, in Table~\ref{tab:tempi} we report the average computational time for a single fold with the optimal parameters on the four largest datasets: CAS, AIDS, NCI1, GDD. The parameters influencing the speed of the kernel are reported between brackets. {In this case, we reported the times corresponding to all the considered kernels.}
 Even when comparing the executions related to the optimal parameters, ODD$_\nuovotk$ is faster or comparable to NSPDK {and ODD$_{ST_h}$ is faster or comparable to FS}. 

%
%
\subsection{Experiments on full NCI datasets}
{In this set of experiments, we analyze how the proposed kernels and the competitors scale up with bigger datasets. We considered two datasets, NCI123 and NCI\_AIDS, each one with more than $40,000$ examples (see Table~\ref{tab:datasets}).\\
In NCI123\footnote{http://pubchem.ncbi.nlm.nih.gov/bioassay/123} the growth inhibition of the MOLT-4 human Leukemia tumor cell line is measured as a screen for
anti-cancer activity. For each compound an activity score  of -LogGI50 is measured, where
GI50 is the concentration of the compound required for 50\% inhibition of
tumor growth. A compound is classified as active (positive class) or inactive
(negative class) if the activity score is, respectively, above or below a specified
threshold. The dataset is composed by 40,952 examples.
NCI\_AIDS\footnote{http://wiki.nci.nih.gov/display/NCIDTPdata/AIDS+Antiviral+Screen+Data} is an  anti-HIV database that contains  42,682  molecules, experimentally detected to protect (confirmed active), moderately  protect (confirmed moderate) or not protect (inactive) the CEM cells from HIV-1  infection.
From these classes we derived a binary classification problem, i.e. distinguishing inactive from confirmed and moderately protective molecules.}

{Since these two datasets are unbalanced, for this set of experiments we adopted the Area Under the  Receiver Operating Characteristic curve (AUROC or AUC) as performance measure, since it is suited for unbalanced datasets.
The experimental setup in this case is different w.r.t. the one presented in Section~\ref{expsold}. Indeed, when the number of examples is large, computing the Gram matrix is unfeasible.
In this case, for each considered kernel configuration, we computed the explicit features (memorized in a sparse format) associated to each example. With this explicit feature representation, it is possible to train a linear SVM\footnote{In our implementation we adopted \textit{Liblinear~\citep{Fan2008}}.}. Note that the computed solution is equivalent to the one that can be found by a C-SVM applied to the kernel matrix generated by the graph kernel. However, in this way it is possible to handle very large datasets in a reasonable amount of time.} 
{A 10x10 CV test with confidence level 95\% (and 10 degrees of freedom) has been executed between each pair of kernels on the two datasets \cite{Japkowicz:2011:ELA:1964882}. } 
\begin{table}[t]
\centering
{
\begin{tabular}{|l|c|c|}
\hline
\scriptsize $Kernel$  	& \scriptsize NCI123 & \scriptsize NCI\_AIDS \\
\hline
\scriptsize Graphlet & 54.93{\bfseries *} (7) \vspace{-0.1cm}& $67.74${\bfseries *} (7)\\
 &  {\tiny $\pm 0.24$}& {\tiny $\pm 0.15$} \\
\scriptsize FS & $61.08${\bfseries *} (6) \vspace{-0.1cm}& $83.73${\bfseries *} (5) \\
 &  {\tiny $\pm 0.34$}& {\tiny $\pm 0.17$} \\
\scriptsize NSPDK		& $62.45$  (3) \vspace{-0.1cm}& $83.80${\bfseries *} (3)\\
& {\tiny $\pm 0.39$} & {\tiny $\pm 0.23$} \\
\scriptsize ODD$_{\text{ST}_h}$	& $62.11$  (4) & $83.77${\bfseries *} (4) \\
& {\tiny $\pm 0.30$} & {\tiny $\pm 0.22$}  \\
\scriptsize ODD$_{\text{ST}_h}^{\text{\tiny{TANH}}}$	& $62.76$  (2) \vspace{-0.1cm}& $85.56$ (2) \\
 &{\tiny $\pm 0.21$}& {\tiny $\pm 0.23$} \\
\scriptsize ODD$_{\text{ST}_+}$	& $61.70${\bfseries *} (5) \vspace{-0.1cm}& $83.36${\bfseries *} (6) \\
 & {\tiny $\pm 0.36$}& {\tiny $\pm 0.30$} \\
\scriptsize ODD$_{\text{ST}_+}^{\text{\tiny{TANH}}}$		& {\bfseries 63.20} (1) \vspace{-0.1cm}& {\bfseries85.64} (1)\\
& {\tiny $\pm 0.29$} & {\tiny $\pm 0.15$} \\
\hline
\end{tabular}
 \caption{Average AUC results $\pm$ standard deviation in nested 10-fold cross validation for the Graphlet, the Fast Subtree, the Neighborhood Subgraph Pairwise Distance, the $ODD_{ST_h}$, the $ODD_{\nuovotk}$, the ODD$_{\text{ST}_h}^{\text{\tiny{TANH}}}$ and the $ODD_{\text{ST}_+}^{\text{\tiny{TANH}}}$	
  kernels obtained on NCI123 and NCI\_AIDS  datasets. The rank of the kernel is reported between brackets. The symbol {\bfseries *} denotes the kernels whose performance difference with respect to the top-ranked kernel is statistically significant. \label{tab:nestedkfoldresultsBig}}
  }
\end{table}
{Table~\ref{tab:nestedkfoldresultsBig} reports the AUC results obtained, for the two considered datasets, 
by kernels for which it is possible to generate the explicit feature space representation of input examples. 
The combination of the techniques proposed in the paper, $\text{ST}_+$ and \textit{tanh}, leads to best performances on both datasets. 
The performance difference between ODD$_{\text{ST}_+}$ e ODD$_{\text{ST}_+}^{\text{\tiny{TANH}}}$ is statistically significant on both datasets. 
The use of \textit{tanh} yields statistically significant improved performances for ODD$_{\text{ST}_+}^{\text{\tiny{TANH}}}$ on NCI\_AIDS with respect to all other kernels except ODD$_{\text{ST}_h}^{\text{\tiny{TANH}}}$.} 

{Figure~\ref{fig:tempinci123} reports the average computational time required to perform the learning procedure for a fixed kernel, as a function of the $h$ parameter, for the NCI123 and NCI\_AIDS datasets. This procedure comprehends the feature generation step, and the training phase of the linear SVM model. We decided to report the overall times here because the run-times of linear SVM depends on the characteristics of the kernel, and thus comparing only the feature generation part would not be fair.
With the considered learning procedure, the number of non-zero features generated by the kernel influences the total run-time. Indeed, the FS kernel is the fastest one, being the one that generates the smallest number of features.
The time required by the training procedure grows almost linearly for ODD$_{\text{ST}_h}$, ODD$_{\text{ST}_h}^{\text{\tiny{TANH}}}$ and ODD$_{\nuovotk}$, while it grows more than linearly for ODD$_{\nuovotk}^{\text{\tiny{TANH}}}$. Note, however, that ODD$_{\nuovotk}^{\text{\tiny{TANH}}}$ is still faster than NSPDK. 
It is interesting to note that NSPDK with $d=1$ is slower than NSPDK with $d=7$ on NCI123, even if the latter has a larger feature space. In this case, probably the former kernel is less discriminative and thus the corresponding optimization problem that the linear SVM must solve is more difficult.
}\\
\begin{figure}[t]
	\centering
	\includegraphics[width=0.9\columnwidth]{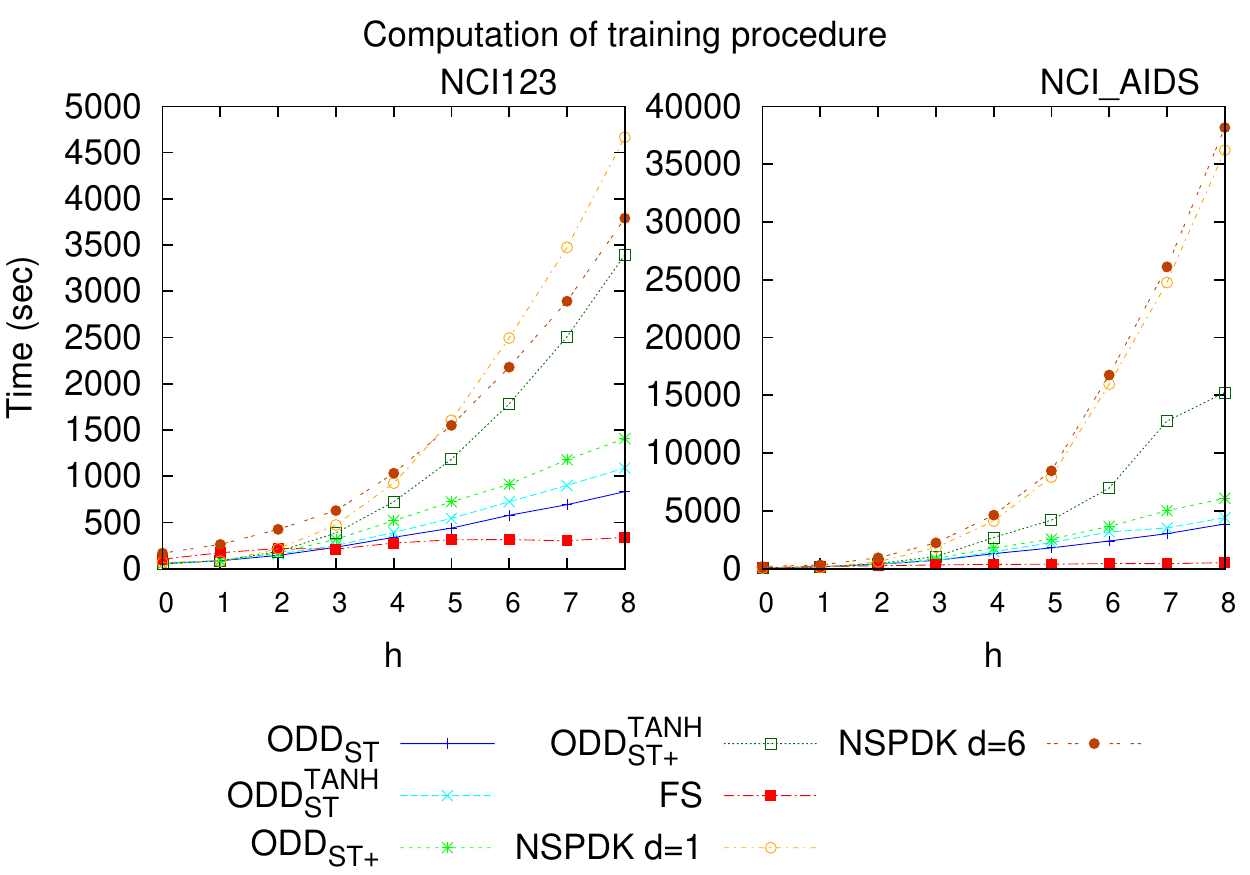}
	\caption{{Time needed to perform all the training procedure, as a function of $h$, for all the considered kernels on NCI123 (left) and NCI\_AIDS (right) datasets.\label{fig:tempinci123}}}
\end{figure}
\begin{table}[p]
\centering
{
\begin{tabular}{|l|c|c|}
\hline
\scriptsize $Kernel$  	& \scriptsize NCI123 & \scriptsize NCI\_AIDS \\
\hline
\scriptsize Graphlet & 698 \vspace{-0.1cm}& 1772\\
 &  {\tiny (C=0.01)}& {\tiny (C=0.001)} \\
\scriptsize FS &  261\vspace{-0.1cm}& 237 \\
 &  {\tiny (h=4,C=0.1)}& {\tiny (h=3,C=0.1)} \\
\scriptsize NSPDK \vspace{-0.1cm}& 246 &1240 \\
& {\tiny (h=2,d=5,C=1)} & {\tiny (h=3,d=6,C=1)} \\
\scriptsize ODD$_{\text{ST}_h}$	& 850 & 1608 \\
& {\tiny (h=7,C=100)} & {\tiny (h=5,C=100)}  \\
\scriptsize ODD$_{\text{ST}_h}^{\text{\tiny{TANH}}}$	& 692  \vspace{-0.1cm}& 2219 \\
 &{\tiny (h=6,C=1)}& {\tiny (h=8,C=1)} \\
\scriptsize ODD$_{\text{ST}_+}$	& 924 \vspace{-0.1cm}& 790  \\
 & {\tiny (h=5,C=10)}& {\tiny (h=3,C=10)} \\
\scriptsize ODD$_{\text{ST}_+}^{\text{\tiny{TANH}}}$		& 694 \vspace{-0.1cm}& 7739 \\
& {\tiny (h=4,C=1)} & {\tiny (h=8,C=1)} \\
\hline
\end{tabular}
 \caption{Time needed to perform all the training procedure with the optimal parameter configuration (reported between brackets) for all the considered kernels on NCI123 and NCI\_AIDS datasets.\label{tab:TimesBig}}
  }
\end{table}
{Table~\ref{tab:TimesBig} reports the computational time required to compute the different kernels with the optimal parameters obtained by a 10-fold cross validation. Note that higher computational times generally corresponds to higher values for the optimal $h$ parameter.}\\
\begin{figure}[p]
	\centering
	{
	\includegraphics[width=0.8\columnwidth]{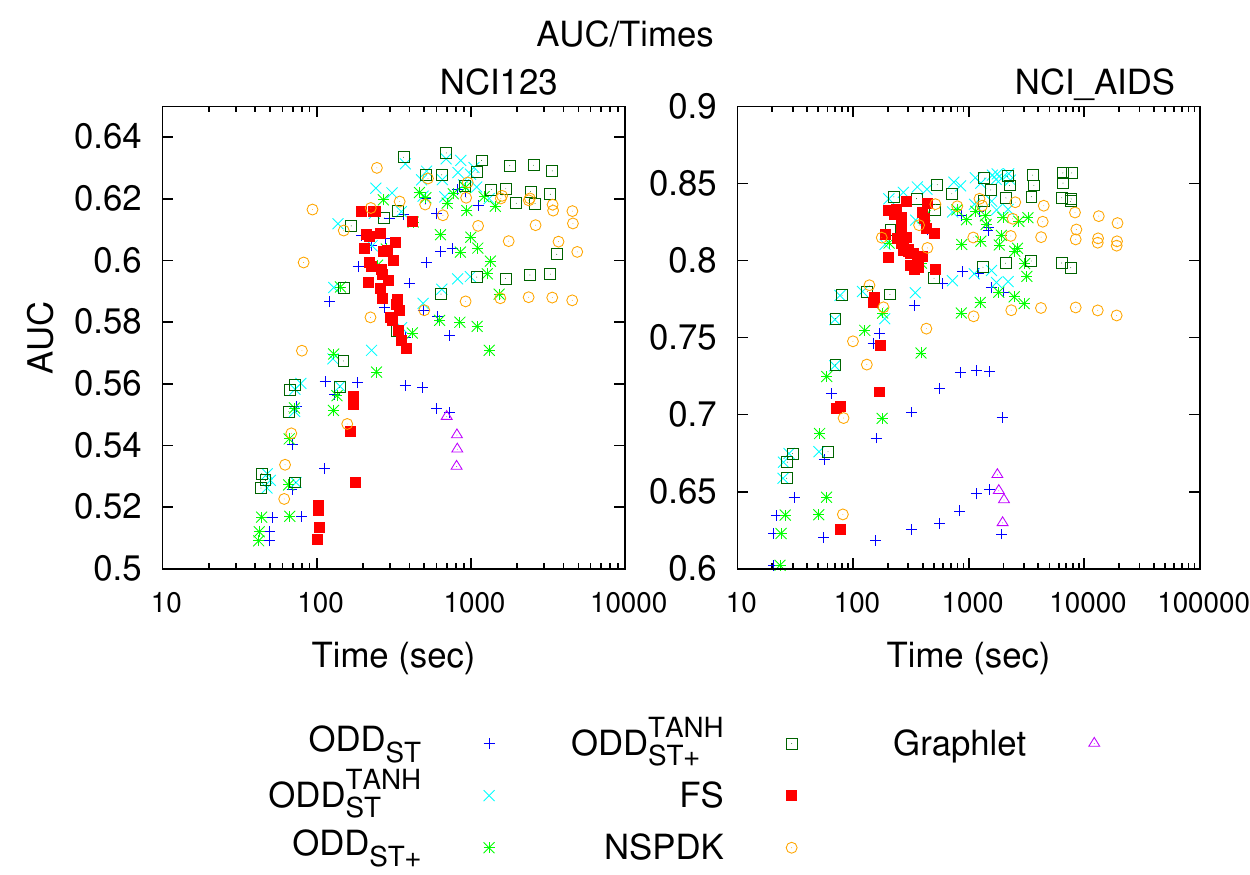}
	\caption{{Relationship between the AUC (obtaindes in 10-fold cross validation) and the time needed to perform all the training procedure. A point is reported for each $h$ and $C$ parameters combination, for all the considered kernels on NCI123 (left) and NCI\_AIDS (right) dataset.Note that the $x$ axis is in log scale.}\label{fig:tempiAUC}}
	}
	\end{figure}
{On the considered datasets, higher AUC corresponds to higher computational times for the respective kernel. 
It is interesting to analyze the relationship between AUC values and running times for non-optimal parameters, i.e. to understand which kernel is the most convenient if there is a strict time constraint to comply to. 
 Figure~\ref{fig:tempiAUC} plots the performances of the different kernels with respect to the time required to perform the training procedure, for NCI123 and NCI\_AIDS datasets.
In NCI123 dataset, ODD$_{\text{ST}_h}^{\text{\tiny{TANH}}}$ and ODD$_{\nuovotk}^{\text{\tiny{TANH}}}$ have the highest points in the plot starting from approximatively a runtime of 400 seconds. Below that computational time, the NSPDK is the best performing kernel. 
On the other hand, on NCI\_AIDS dataset, ODD$_{\text{ST}_h}^{\text{\tiny{TANH}}}$ and ODD$_{\nuovotk}^{\text{\tiny{TANH}}}$ are the better performing kernels for almost every time threshold.}

\section{Conclusions and future works} \label{sec:conclusions}

\textcolor{black}{The contribution of this paper is twofold.
First, we propose a novel instance of the ODD graph kernel based on a novel tree kernel, $\nuovotk$. 
This constitutes an example of how the generality of the framework can 
potentially lead to the definition of novel graph kernels that can improve the state-of-the-art.
Second, we define a novel, non-linear, feature weighting scheme for the ODD kernels, that can in principle be applied to any graph kernel with an explicit feature space representation.
As a future work, we plan to apply this and other weighting schemes also to other state-of-the-art graph kernels.
The experimental results show that the proposed kernels have state of the art performances on six benchmark graph datasets from bioinformatics, and on two graph datasets for image classification.}
{Moreover, experiments on two large graph datasets show that our approach is able to scale up to real-world sized datasets.}
  
\section*{Acknowledgments}
This work was supported by the University of Padova under the strategic project \textit{BIOINFOGEN}.

\section*{References}

\bibliographystyle{elsarticle-num-names}
\bibliography{Bibliografia/Mendeley,Bibliografia/swork}

\begin{thebibliography}{46}
\providecommand{\natexlab}[1]{#1}
\providecommand{\url}[1]{\texttt{#1}}
\providecommand{\urlprefix}{URL }
\expandafter\ifx\csname urlstyle\endcsname\relax
  \providecommand{\doi}[1]{doi:\discretionary{}{}{}#1}\else
  \providecommand{\doi}[1]{doi:\discretionary{}{}{}\begingroup
  \urlstyle{rm}\url{#1}\endgroup}\fi
\providecommand{\bibinfo}[2]{#2}

\bibitem[{Denoyer and Gallinari(2007)}]{Denoyer2007}
\bibinfo{author}{L.~Denoyer}, \bibinfo{author}{P.~Gallinari},
  \bibinfo{title}{{Report on the XML mining track at INEX 2005 and INEX 2006:
  categorization and clustering of XML documents}}, \bibinfo{journal}{SIGIR
  Forum} \bibinfo{volume}{41}~(\bibinfo{number}{1}) (\bibinfo{year}{2007})
  \bibinfo{pages}{79--90}, ISSN \bibinfo{issn}{0163-5840},
  \doi{\bibinfo{doi}{http://doi.acm.org/10.1145/1273221.1273230}}.

\bibitem[{Dobson and Doig(2003)}]{dobson2003}
\bibinfo{author}{P.~D. Dobson}, \bibinfo{author}{A.~J. Doig},
  \bibinfo{title}{{Distinguishing Enzyme Structures from Non-enzymes Without
  Alignments}}, \bibinfo{journal}{Journal of Molecular Biology}
  \bibinfo{volume}{330}~(\bibinfo{number}{4}) (\bibinfo{year}{2003})
  \bibinfo{pages}{771--783}, ISSN \bibinfo{issn}{0022-2836},
  \doi{\bibinfo{doi}{10.1016/S0022-2836(03)00628-4}}.

\bibitem[{Wale et~al.(2008)Wale, Watson, and
  Karypis}]{springerlink:10.1007/s10115-007-0103-5}
\bibinfo{author}{N.~Wale}, \bibinfo{author}{I.~Watson},
  \bibinfo{author}{G.~Karypis}, \bibinfo{title}{{Comparison of descriptor
  spaces for chemical compound retrieval and classification}},
  \bibinfo{journal}{Knowledge and Information Systems}
  \bibinfo{volume}{14}~(\bibinfo{number}{3}) (\bibinfo{year}{2008})
  \bibinfo{pages}{347--375}, ISSN \bibinfo{issn}{0219-1377}.

\bibitem[{Weislow et~al.(1989)Weislow, Kiser, Fine, Bader, Shoemaker, and
  Boyd}]{Weislow1989}
\bibinfo{author}{O.~S. Weislow}, \bibinfo{author}{R.~Kiser},
  \bibinfo{author}{D.~L. Fine}, \bibinfo{author}{J.~Bader},
  \bibinfo{author}{R.~H. Shoemaker}, \bibinfo{author}{M.~R. Boyd},
  \bibinfo{title}{{New soluble-formazan assay for HIV-1 cytopathic effects:
  application to high-flux screening of synthetic and natural products for
  AIDS-antiviral activity.}}, \bibinfo{journal}{Journal of the National Cancer
  Institute} \bibinfo{volume}{81}~(\bibinfo{number}{8}) (\bibinfo{year}{1989})
  \bibinfo{pages}{577--586}, ISSN \bibinfo{issn}{0027-8874}.

\bibitem[{Shawe-Taylor and Cristianini(2004)}]{Taylor-Cristianini:Book2004}
\bibinfo{author}{J.~Shawe-Taylor}, \bibinfo{author}{N.~Cristianini},
  \bibinfo{title}{{Kernel Methods for Pattern Analysis}},
  \bibinfo{publisher}{Cambridge University Press}, \bibinfo{address}{New York,
  NY, USA}, ISBN \bibinfo{isbn}{0521813972}, \bibinfo{year}{2004}.

\bibitem[{Sim{\~{o}}es et~al.(2013)Sim{\~{o}}es, Galhardas, and
  Matos}]{Simoes2013}
\bibinfo{author}{G.~Sim{\~{o}}es}, \bibinfo{author}{H.~Galhardas},
  \bibinfo{author}{D.~Matos}, \bibinfo{title}{{A Labeled Graph Kernel for
  Relationship Extraction}}, in: \bibinfo{booktitle}{CoRR},
  \urlprefix\url{http://arxiv.org/abs/1302.4874}, \bibinfo{year}{2013}.

\bibitem[{Vries(2013)}]{Vries2013a}
\bibinfo{author}{G.~D. Vries}, \bibinfo{title}{{Graph Kernels for Task 1 and 2
  of the Linked Data Data-Mining Challenge 2013}}, in:
  \bibinfo{booktitle}{DMoLD}, \bibinfo{year}{2013}.

\bibitem[{Wang and Sahbi(2013)}]{Wang2013a}
\bibinfo{author}{L.~Wang}, \bibinfo{author}{H.~Sahbi},
  \bibinfo{title}{{Directed Acyclic Graph Kernels for Action Recognition}},
  \bibinfo{journal}{2013 IEEE International Conference on Computer Vision}
  (\bibinfo{year}{2013})
  \bibinfo{pages}{3168--3175}\doi{\bibinfo{doi}{10.1109/ICCV.2013.393}}.

\bibitem[{Bleik et~al.(2013)Bleik, Mishra, Huan, and Song}]{Bleik2013}
\bibinfo{author}{S.~Bleik}, \bibinfo{author}{M.~Mishra},
  \bibinfo{author}{J.~Huan}, \bibinfo{author}{M.~Song}, \bibinfo{title}{{Text
  categorization of biomedical data sets using graph kernels and a controlled
  vocabulary.}}, \bibinfo{journal}{IEEE/ACM transactions on computational
  biology and bioinformatics / IEEE, ACM}
  \bibinfo{volume}{10}~(\bibinfo{number}{5}) (\bibinfo{year}{2013})
  \bibinfo{pages}{1211--7}, ISSN \bibinfo{issn}{1557-9964},
  \doi{\bibinfo{doi}{10.1109/TCBB.2013.16}}.

\bibitem[{Kundu et~al.(2013)Kundu, Costa, and Backofen}]{Kundu2013}
\bibinfo{author}{K.~Kundu}, \bibinfo{author}{F.~Costa},
  \bibinfo{author}{R.~Backofen}, \bibinfo{title}{{A graph kernel approach for
  alignment-free domain-peptide interaction prediction with an application to
  human SH3 domains.}}, \bibinfo{journal}{Bioinformatics (Oxford, England)}
  \bibinfo{volume}{29}~(\bibinfo{number}{13}) (\bibinfo{year}{2013})
  \bibinfo{pages}{i335--43}, ISSN \bibinfo{issn}{1367-4811},
  \doi{\bibinfo{doi}{10.1093/bioinformatics/btt220}}.

\bibitem[{Cesa-Bianchi et~al.(2005)Cesa-Bianchi, Conconi, and
  Gentile}]{doi:10.1137/S0097539703432542}
\bibinfo{author}{N.~Cesa-Bianchi}, \bibinfo{author}{A.~Conconi},
  \bibinfo{author}{C.~Gentile}, \bibinfo{title}{{A Second-Order Perceptron
  Algorithm}}, \bibinfo{journal}{SIAM Journal on Computing}
  \bibinfo{volume}{34}~(\bibinfo{number}{3}) (\bibinfo{year}{2005})
  \bibinfo{pages}{640--668}.

\bibitem[{Haussler(1999)}]{Haussler1999}
\bibinfo{author}{D.~Haussler}, \bibinfo{title}{{Convolution Kernels on Discrete
  Structures}}, \bibinfo{type}{Tech. Rep.}, \bibinfo{institution}{Department of
  Computer Science, University of California at Santa Cruz},
  \bibinfo{year}{1999}.

\bibitem[{Collins and Duffy(2002)}]{Collins2002}
\bibinfo{author}{M.~Collins}, \bibinfo{author}{N.~Duffy}, \bibinfo{title}{{New
  ranking algorithms for parsing and tagging: kernels over discrete structures,
  and the voted perceptron}}, in: \bibinfo{booktitle}{Proceedings of the 40th
  Annual Meeting on Association for Computational Linguistics},
  \bibinfo{publisher}{Association for Computational Linguistics},
  \bibinfo{address}{Philadelphia, Pennsylvania}, \bibinfo{pages}{263--270},
  \bibinfo{year}{2002}.

\bibitem[{Vishwanathan and Smola(2003)}]{Vishwanathan2003}
\bibinfo{author}{S.~V.~N. Vishwanathan}, \bibinfo{author}{A.~J. Smola},
  \bibinfo{title}{Fast kernels for string and tree matching}, in:
  \bibinfo{booktitle}{Advances in Neural Information Processing Systems 15},
  \bibinfo{publisher}{MIT Press}, \bibinfo{pages}{569--576},
  \bibinfo{year}{2003}.

\bibitem[{Moschitti(2006)}]{Moschitti2006a}
\bibinfo{author}{A.~Moschitti}, \bibinfo{title}{{Efficient Convolution Kernels
  for Dependency and Constituent Syntactic Trees}}, in:
  \bibinfo{booktitle}{ECML}, vol. \bibinfo{volume}{4212} of
  \emph{\bibinfo{series}{Lecture Notes in Computer Science}}, ISBN
  \bibinfo{isbn}{3-540-45375-X}, \bibinfo{pages}{318--329},
  \bibinfo{year}{2006}.

\bibitem[{Aiolli et~al.(2009)Aiolli, {Da San Martino}, and
  Sperduti}]{Aiolli2009}
\bibinfo{author}{F.~Aiolli}, \bibinfo{author}{G.~{Da San Martino}},
  \bibinfo{author}{A.~Sperduti}, \bibinfo{title}{{Route kernels for trees}},
  in: \bibinfo{booktitle}{Proceedings of the 26th Annual International
  Conference on Machine Learning - ICML '09}, \bibinfo{publisher}{ACM Press},
  \bibinfo{address}{New York, New York, USA}, ISBN
  \bibinfo{isbn}{9781605585161}, \bibinfo{pages}{17--24},
  \doi{\bibinfo{doi}{10.1145/1553374.1553377}}, \bibinfo{year}{2009}.

\bibitem[{Aiolli et~al.(2011)Aiolli, {Da San Martino}, and
  Sperduti}]{Aiolli2011}
\bibinfo{author}{F.~Aiolli}, \bibinfo{author}{G.~{Da San Martino}},
  \bibinfo{author}{A.~Sperduti}, \bibinfo{title}{{Extending Tree Kernels with
  Topological Information}}, \bibinfo{journal}{ICANN} \bibinfo{volume}{6791}
  (\bibinfo{year}{2011}) \bibinfo{pages}{142--149}.

\bibitem[{Bacciu et~al.(2012)Bacciu, Micheli, and
  Sperduti}]{DBLP:conf/icann/BacciuMS12}
\bibinfo{author}{D.~Bacciu}, \bibinfo{author}{A.~Micheli},
  \bibinfo{author}{A.~Sperduti}, \bibinfo{title}{A Generative Multiset Kernel
  for Structured Data}, in: \bibinfo{editor}{A.~E.~P. Villa},
  \bibinfo{editor}{W.~Duch}, \bibinfo{editor}{P.~{\'E}rdi},
  \bibinfo{editor}{F.~Masulli}, \bibinfo{editor}{G.~Palm} (Eds.),
  \bibinfo{booktitle}{ICANN (1)}, vol. \bibinfo{volume}{7552} of
  \emph{\bibinfo{series}{Lecture Notes in Computer Science}},
  \bibinfo{publisher}{Springer}, ISBN \bibinfo{isbn}{978-3-642-33268-5},
  \bibinfo{pages}{57--64}, \bibinfo{year}{2012}.

\bibitem[{Gartner et~al.(2003)Gartner, Flach, Wrobel, and
  G{\"{a}}rtner}]{Gartner2003a}
\bibinfo{author}{T.~Gartner}, \bibinfo{author}{P.~Flach},
  \bibinfo{author}{S.~Wrobel}, \bibinfo{author}{T.~G{\"{a}}rtner},
  \bibinfo{title}{{On Graph Kernels: Hardness Results and Efficient
  Alternatives}}, in: \bibinfo{editor}{B.~Sch{\"{o}}lkopf},
  \bibinfo{editor}{M.~K. Warmuth} (Eds.), \bibinfo{booktitle}{Proceedings of
  the 16th Annual Conference on Computational Learning Theory and 7th Kernel
  Workshop}, vol. \bibinfo{volume}{2777} of \emph{\bibinfo{series}{Lecture
  Notes in Computer Science}}, \bibinfo{publisher}{Springer Berlin Heidelberg},
  \bibinfo{address}{Berlin, Heidelberg}, ISBN
  \bibinfo{isbn}{978-3-540-40720-1}, \bibinfo{pages}{129--143},
  \doi{\bibinfo{doi}{10.1007/b12006}}, \bibinfo{year}{2003}.

\bibitem[{Schietgat et~al.(2009)Schietgat, Costa, Ramon, and {De
  Raedt}}]{Schietgat2009}
\bibinfo{author}{L.~Schietgat}, \bibinfo{author}{F.~Costa},
  \bibinfo{author}{J.~Ramon}, \bibinfo{author}{L.~{De Raedt}},
  \bibinfo{title}{{Maximum common subgraph mining: a fast and effective
  approach towards feature generation}}, in: \bibinfo{booktitle}{7th
  International Workshop on Mining and Learning with Graphs},
  \bibinfo{pages}{1--3}, \bibinfo{year}{2009}.

\bibitem[{Costa and {De Grave}(2010)}]{Costa2010}
\bibinfo{author}{F.~Costa}, \bibinfo{author}{K.~{De Grave}},
  \bibinfo{title}{{Fast neighborhood subgraph pairwise distance kernel}}, in:
  \bibinfo{editor}{J.~F. Joachims}, \bibinfo{editor}{Thorsten} (Eds.),
  \bibinfo{booktitle}{Proceedings of the 27th International Conference on
  Machine Learning (ICML-10)}, \bibinfo{publisher}{Omnipress},
  \bibinfo{pages}{255--262}, \bibinfo{year}{2010}.

\bibitem[{Suard et~al.(2007)Suard, Rakotomamonjy, and Bensrhair}]{Suard2007}
\bibinfo{author}{F.~Suard}, \bibinfo{author}{a.~Rakotomamonjy},
  \bibinfo{author}{a.~Bensrhair}, \bibinfo{title}{{Kernel on bag of paths for
  measuring similarity of shapes}}, \bibinfo{journal}{European Symposium on
  Artificial Neural Networks}  (\bibinfo{year}{2007}) \bibinfo{pages}{1--6}.

\bibitem[{Mah{\'e} and Vert(2009)}]{Mah'e2009}
\bibinfo{author}{P.~Mah{\'e}}, \bibinfo{author}{J.~Vert}, \bibinfo{title}{Graph
  kernels based on tree patterns for molecules}, \bibinfo{journal}{Machine
  Learning} \bibinfo{volume}{75}~(\bibinfo{number}{1}) (\bibinfo{year}{2009})
  \bibinfo{pages}{3--35}.

\bibitem[{Shervashidze and Borgwardt(2009)}]{Shervashidze2009a}
\bibinfo{author}{N.~Shervashidze}, \bibinfo{author}{K.~M. Borgwardt},
  \bibinfo{title}{Fast subtree kernels on graphs}, in:
  \bibinfo{booktitle}{NIPS}, \bibinfo{pages}{1660--1668}, \bibinfo{year}{2009}.

\bibitem[{Collins and Duffy(2001)}]{Collins2001}
\bibinfo{author}{M.~Collins}, \bibinfo{author}{N.~Duffy},
  \bibinfo{title}{{Convolution Kernels for Natural Language}}, in:
  \bibinfo{editor}{T.~G. Dietterich}, \bibinfo{editor}{S.~Becker},
  \bibinfo{editor}{Z.~Ghahramani} (Eds.), \bibinfo{booktitle}{NIPS},
  \bibinfo{publisher}{MIT Press}, \bibinfo{pages}{625--632},
  \bibinfo{year}{2001}.

\bibitem[{{Da San Martino} et~al.(2015){Da San Martino}, Navarin, and
  Sperduti}]{DaSanMartino2015}
\bibinfo{author}{G.~{Da San Martino}}, \bibinfo{author}{N.~Navarin},
  \bibinfo{author}{A.~Sperduti}, \bibinfo{title}{{Exploiting the ODD framework
  to define a novel effective graph kernel.}}, in:
  \bibinfo{booktitle}{proceedings of the 23th European Symposium on Artificial
  Neural Networks, Computational Intelligence and Machine Learning},
  \bibinfo{year}{2015}.

\bibitem[{{Da San Martino} et~al.(2012{\natexlab{a}}){Da San Martino}, Navarin,
  and Sperduti}]{Dasan2012}
\bibinfo{author}{G.~{Da San Martino}}, \bibinfo{author}{N.~Navarin},
  \bibinfo{author}{A.~Sperduti}, \bibinfo{title}{{A Tree-Based Kernel for
  Graphs}}, in: \bibinfo{booktitle}{Proceedings of the Twelfth SIAM
  International Conference on Data Mining}, \bibinfo{pages}{975--986},
  \bibinfo{year}{2012}{\natexlab{a}}.

\bibitem[{{Da San Martino} et~al.(2012{\natexlab{b}}){Da San Martino}, Navarin,
  and Sperduti}]{DaSanMartino2012}
\bibinfo{author}{G.~{Da San Martino}}, \bibinfo{author}{N.~Navarin},
  \bibinfo{author}{A.~Sperduti}, \bibinfo{title}{{A memory efficient graph
  kernel}}, in: \bibinfo{booktitle}{the 2012 International Joint Conference on
  Neural Networks (IJCNN)}, \bibinfo{publisher}{IEEE},
  \bibinfo{year}{2012}{\natexlab{b}}.

\bibitem[{Yanardag and Vishwanathan(2014)}]{Yanardag2014}
\bibinfo{author}{P.~Yanardag}, \bibinfo{author}{S.~V.~N. Vishwanathan},
  \bibinfo{title}{{The Structurally Smoothed Graphlet Kernel}},
  \bibinfo{journal}{arXiv} .

\bibitem[{Kazius et~al.(2006)Kazius, Nijssen, Kok, Back, and
  Ijzerman}]{Kazius2006}
\bibinfo{author}{J.~Kazius}, \bibinfo{author}{S.~Nijssen},
  \bibinfo{author}{J.~Kok}, \bibinfo{author}{T.~Back}, \bibinfo{author}{A.~P.
  Ijzerman}, \bibinfo{title}{{Substructure Mining Using Elaborate Chemical
  Representation}}, \bibinfo{journal}{J. Chem. Inf. Model.}
  \bibinfo{volume}{46}~(\bibinfo{number}{2}) (\bibinfo{year}{2006})
  \bibinfo{pages}{597--605}.

\bibitem[{Bringmann et~al.(2006)Bringmann, Zimmermann, Raedt, and
  Nijssen}]{DBLP:conf/pkdd/BringmannZRN06}
\bibinfo{author}{B.~Bringmann}, \bibinfo{author}{A.~Zimmermann},
  \bibinfo{author}{L.~D. Raedt}, \bibinfo{author}{S.~Nijssen},
  \bibinfo{title}{{Don't Be Afraid of Simpler Patterns}}, in:
  \bibinfo{editor}{J.~F{\"{u}}rnkranz}, \bibinfo{editor}{T.~Scheffer},
  \bibinfo{editor}{M.~Spiliopoulou} (Eds.), \bibinfo{booktitle}{PKDD}, vol.
  \bibinfo{volume}{4213} of \emph{\bibinfo{series}{Lecture Notes in Computer
  Science}}, \bibinfo{publisher}{Springer}, ISBN \bibinfo{isbn}{3-540-45374-1},
  \bibinfo{pages}{55--66}, \bibinfo{year}{2006}.

\bibitem[{Helma et~al.(2004)Helma, Cramer, Kramer, and {De Raedt}}]{Helma2004}
\bibinfo{author}{C.~Helma}, \bibinfo{author}{T.~Cramer},
  \bibinfo{author}{S.~Kramer}, \bibinfo{author}{L.~{De Raedt}},
  \bibinfo{title}{{Data mining and machine learning techniques for the
  identification of mutagenicity inducing substructures and structure activity
  relationships of noncongeneric compounds}}, \bibinfo{journal}{Journal of
  Chemical Information and Computer Sciences}
  \bibinfo{volume}{44}~(\bibinfo{number}{4}) (\bibinfo{year}{2004})
  \bibinfo{pages}{1402--1411}, ISSN \bibinfo{issn}{0095-2338},
  \doi{\bibinfo{doi}{10.1021/ci034254q}}.

\bibitem[{Saigo et~al.(2009)Saigo, Nowozin, Kadowaki, Kudo, and
  Tsuda}]{Saigo2009}
\bibinfo{author}{H.~Saigo}, \bibinfo{author}{S.~Nowozin},
  \bibinfo{author}{T.~Kadowaki}, \bibinfo{author}{T.~Kudo},
  \bibinfo{author}{K.~Tsuda}, \bibinfo{title}{gBoost: a mathematical
  programming approach to graph classification and regression.},
  \bibinfo{journal}{Machine Learning}  (\bibinfo{year}{2009})
  \bibinfo{pages}{69--89}.

\bibitem[{Kashima et~al.(2003)Kashima, Tsuda, and Inokuchi}]{Kashima2003}
\bibinfo{author}{H.~Kashima}, \bibinfo{author}{K.~Tsuda},
  \bibinfo{author}{A.~Inokuchi}, \bibinfo{title}{Marginalized Kernels Between
  Labeled Graphs.}, in: \bibinfo{editor}{T.~Fawcett},
  \bibinfo{editor}{N.~Mishra} (Eds.), \bibinfo{booktitle}{ICML},
  \bibinfo{publisher}{AAAI Press}, ISBN \bibinfo{isbn}{1-57735-189-4},
  \bibinfo{pages}{321--328}, \bibinfo{year}{2003}.

\bibitem[{Mah{\'e} et~al.(2004)Mah{\'e}, Ueda, Akutsu, Perret, and
  Vert}]{Mah'e2004}
\bibinfo{author}{P.~Mah{\'e}}, \bibinfo{author}{N.~Ueda},
  \bibinfo{author}{T.~Akutsu}, \bibinfo{author}{J.~Perret},
  \bibinfo{author}{J.~Vert}, \bibinfo{title}{Extensions of marginalized graph
  kernels}, in: \bibinfo{booktitle}{Proceedings of the twenty-first
  international conference on Machine learning}, \bibinfo{publisher}{{ACM}},
  \bibinfo{address}{Banff, Alberta, Canada}, \bibinfo{pages}{70},
  \bibinfo{year}{2004}.

\bibitem[{Vishwanathan et~al.(2006)Vishwanathan, Borgwardt, and
  Schraudolph}]{Vishwanathan2006}
\bibinfo{author}{S.~V.~N. Vishwanathan}, \bibinfo{author}{K.~M. Borgwardt},
  \bibinfo{author}{N.~N. Schraudolph}, \bibinfo{title}{{Fast Computation of
  Graph Kernels}}, in: \bibinfo{booktitle}{NIPS}, \bibinfo{pages}{1449--1456},
  \bibinfo{year}{2006}.

\bibitem[{Borgwardt and Kriegel(2005)}]{Borgwardt2005}
\bibinfo{author}{K.~M. Borgwardt}, \bibinfo{author}{H.-P. Kriegel},
  \bibinfo{title}{Shortest-Path Kernels on Graphs}, in:
  \bibinfo{booktitle}{Proceedings of the Fifth IEEE International Conference on
  Data Mining}, \bibinfo{publisher}{{IEEE} Computer Society}, ISBN
  \bibinfo{isbn}{0-7695-2278-5}, \bibinfo{pages}{74--81}, \bibinfo{year}{2005}.

\bibitem[{Shervashidze et~al.(2011)Shervashidze, Schweitzer, van Leeuwen,
  Mehlhorn, and Borgwardt}]{Shervashidze2011}
\bibinfo{author}{N.~Shervashidze}, \bibinfo{author}{P.~Schweitzer},
  \bibinfo{author}{E.~J. van Leeuwen}, \bibinfo{author}{K.~Mehlhorn},
  \bibinfo{author}{K.~M. Borgwardt}, \bibinfo{title}{{Weisfeiler-Lehman Graph
  Kernels}}, \bibinfo{journal}{Journal of Machine Learning Research}
  \bibinfo{volume}{12} (\bibinfo{year}{2011}) \bibinfo{pages}{2539--2561}.

\bibitem[{Heinonen et~al.(2012)Heinonen, Rousu, V{\"{a}}lim{\"{a}}ki, and
  M{\"{a}}kinen}]{Heinonen2009}
\bibinfo{author}{M.~Heinonen}, \bibinfo{author}{J.~Rousu},
  \bibinfo{author}{N.~V{\"{a}}lim{\"{a}}ki},
  \bibinfo{author}{V.~M{\"{a}}kinen}, \bibinfo{title}{{Efficient Path Kernels
  for Reaction Function Prediction}}, in: \bibinfo{booktitle}{BIOINFORMATICS
  2012 - Proceedings of the International Conference on Bioinformatics Models,
  Methods and Algorithms}, \bibinfo{pages}{202--207}, \bibinfo{year}{2012}.

\bibitem[{Shervashidze et~al.(2009)Shervashidze, Mehlhorn, Petri, Vishwanathan,
  Borgwardt, Petri, Mehlhorn, and Borgwardt}]{Shervashidze2009}
\bibinfo{author}{N.~Shervashidze}, \bibinfo{author}{K.~Mehlhorn},
  \bibinfo{author}{T.~H. Petri}, \bibinfo{author}{S.~V.~N. Vishwanathan},
  \bibinfo{author}{K.~M. Borgwardt}, \bibinfo{author}{T.~H. Petri},
  \bibinfo{author}{K.~Mehlhorn}, \bibinfo{author}{K.~M. Borgwardt},
  \bibinfo{title}{{Efficient graphlet kernels for large graph comparison}}, in:
  \bibinfo{editor}{D.~van Dyk}, \bibinfo{editor}{M.~Welling} (Eds.),
  \bibinfo{booktitle}{Proceedings of the Twelfth International Conference on
  Artificial Intelligence and Statistics (AISTATS)}, vol.~\bibinfo{volume}{5}
  of \emph{\bibinfo{series}{JMLR: Workshop and Conference Proceedings}},
  \bibinfo{organization}{PASCAL EPrints (United Kingdom)},
  \bibinfo{publisher}{CSAIL}, \bibinfo{address}{Clearwater Beach, Florida,
  USA}, ISBN \bibinfo{isbn}{1938-7228}, \bibinfo{pages}{488--495},
  \bibinfo{year}{2009}.

\bibitem[{Ga{\"{u}}z{\`{e}}re et~al.(2015)Ga{\"{u}}z{\`{e}}re, Grenier, Brun,
  and Villemin}]{Gauzere2015}
\bibinfo{author}{B.~Ga{\"{u}}z{\`{e}}re}, \bibinfo{author}{P.-A. Grenier},
  \bibinfo{author}{L.~Brun}, \bibinfo{author}{D.~Villemin},
  \bibinfo{title}{{Treelet kernel incorporating cyclic, stereo and inter
  pattern information in chemoinformatics}}, \bibinfo{journal}{Pattern
  Recognition} \bibinfo{volume}{48}~(\bibinfo{number}{2})
  (\bibinfo{year}{2015}) \bibinfo{pages}{356--367}, ISSN
  \bibinfo{issn}{0031-3203}.

\bibitem[{Neumann et~al.(2012)Neumann, Patricia, Garnett, and
  Kersting}]{Birlinghoven}
\bibinfo{author}{M.~Neumann}, \bibinfo{author}{N.~Patricia},
  \bibinfo{author}{R.~Garnett}, \bibinfo{author}{K.~Kersting},
  \bibinfo{title}{{Efficient Graph Kernels by Randomization}}, in:
  \bibinfo{editor}{P.~A. Flach}, \bibinfo{editor}{T.~{De Bie}},
  \bibinfo{editor}{N.~Cristianini} (Eds.), \bibinfo{booktitle}{ECML PKDD}, vol.
  \bibinfo{volume}{7523} of \emph{\bibinfo{series}{Lecture Notes in Computer
  Science}}, \bibinfo{publisher}{Springer Berlin Heidelberg},
  \bibinfo{address}{Berlin, Heidelberg}, ISBN
  \bibinfo{isbn}{978-3-642-33459-7}, \bibinfo{pages}{378--393},
  \bibinfo{year}{2012}.

\bibitem[{Pedregosa et~al.(2011)Pedregosa, Varoquaux, Gramfort, Michel,
  Thirion, Grisel, Blondel, Prettenhofer, Weiss, Dubourg, Vanderplas, Passos,
  Cournapeau, Brucher, Perrot, and Duchesnay}]{Pedregosa2012}
\bibinfo{author}{F.~Pedregosa}, \bibinfo{author}{G.~Varoquaux},
  \bibinfo{author}{A.~Gramfort}, \bibinfo{author}{V.~Michel},
  \bibinfo{author}{B.~Thirion}, \bibinfo{author}{O.~Grisel},
  \bibinfo{author}{M.~Blondel}, \bibinfo{author}{P.~Prettenhofer},
  \bibinfo{author}{R.~Weiss}, \bibinfo{author}{V.~Dubourg},
  \bibinfo{author}{J.~Vanderplas}, \bibinfo{author}{A.~Passos},
  \bibinfo{author}{D.~Cournapeau}, \bibinfo{author}{M.~Brucher},
  \bibinfo{author}{M.~Perrot}, \bibinfo{author}{{\'{E}}.~Duchesnay},
  \bibinfo{title}{{Scikit-learn: Machine Learning in Python}},
  \bibinfo{journal}{Journal of Machine Learning Research} \bibinfo{volume}{12}
  (\bibinfo{year}{2011}) \bibinfo{pages}{2825--2830}, ISSN
  \bibinfo{issn}{15324435}.

\bibitem[{Japkowicz and Shah(2011)}]{Japkowicz:2011:ELA:1964882}
\bibinfo{author}{N.~Japkowicz}, \bibinfo{author}{M.~Shah},
  \bibinfo{title}{Evaluating Learning Algorithms: A Classification
  Perspective}, \bibinfo{publisher}{Cambridge University Press},
  \bibinfo{address}{New York, NY, USA}, ISBN \bibinfo{isbn}{0521196000,
  9780521196000}, \bibinfo{year}{2011}.

\bibitem[{Kersting et~al.(2013)Kersting, Mladenov, Garnett, and
  Grohe}]{Kersting2013}
\bibinfo{author}{K.~Kersting}, \bibinfo{author}{M.~Mladenov},
  \bibinfo{author}{R.~Garnett}, \bibinfo{author}{M.~Grohe},
  \bibinfo{title}{{Power Iterated Color Refinement}}, in:
  \bibinfo{booktitle}{28th AAAI Conference on Artificial Intelligence},
  \bibinfo{pages}{1904--1910}, \bibinfo{year}{2013}.

\bibitem[{Fan(2008)}]{Fan2008}
\bibinfo{title}{{LIBLINEAR: A library for large linear classification}},
  \bibinfo{journal}{The Journal of Machine Learning}
  \bibinfo{volume}{9}~(\bibinfo{number}{2008}) (\bibinfo{year}{2008})
  \bibinfo{pages}{1871--1874}.

\end{thebibliography}

\end{document}